\documentclass[]{imsart}

\RequirePackage{amsthm,amsmath,amsfonts,amssymb}
\RequirePackage[authoryear]{natbib}
\RequirePackage[colorlinks,citecolor=blue,urlcolor=blue]{hyperref}
\RequirePackage{graphicx}

\startlocaldefs
\theoremstyle{plain}

\newtheorem{theorem}{theorem}[section]
\newtheorem{lemma}[theorem]{lemma}
\theoremstyle{remark}
\newtheorem{definition}[theorem]{definition}
\newtheorem{remark}[theorem]{remark}
\newtheorem{assumption}[theorem]{assumption}

\usepackage{amsmath,amsthm,amsfonts,amscd,amssymb,bm,mathrsfs}
\usepackage{enumerate}
\usepackage{appendix}
\usepackage{graphicx,xcolor}
\usepackage{epstopdf}
\usepackage{subfigure}
\usepackage[normalem]{ulem}
\usepackage{indentfirst}
\usepackage{cases}
\useunder{\uline}{\ul}{}
\DeclareMathOperator*{\argmin}{arg\,min}
\newcommand{\R}{{\mathbb{R}}}

\newcommand{\E}{{\mathbb{E}}}
\newcommand{\D}{{\mathbb{D}}}

\newcommand{\norm}[1]{{\left\vert\kern-0.25ex\left\vert\kern-0.25ex\left\vert #1
\right\vert\kern-0.25ex\right\vert\kern-0.25ex\right\vert}}

\endlocaldefs

\begin{document}

\begin{frontmatter}
\title{Sparse deep neural networks for nonparametric estimation in high-dimensional sparse regression}
\runtitle{Sparse deep neural networks for nonparametric estimation}

\begin{aug}
\author[A]{\fnms{Dongya}~\snm{Wu}\ead[label=e1]{wudongya@nwu.edu.cn}},
\author[B]{\fnms{Xin}~\snm{Li}\ead[label=e2]{lixin@nwu.edu.cn}}
\address[A]{School of Information Science and Technology, Northwest University, Xi'an 710127, China\printead[presep={,\ }]{e1}}

\address[B]{School of Mathematics, Northwest University, Xi'an 710127, China\printead[presep={,\ }]{e2}}
\end{aug}

\begin{abstract}
Generalization theory has been established for sparse deep neural networks under high-dimensional regime. Beyond generalization, parameter estimation is also important since it is crucial for variable selection and interpretability of deep neural networks. Current theoretical studies concerning parameter estimation mainly focus on two-layer neural networks, which is due to the fact that the convergence of parameter estimation heavily relies on the regularity of the Hessian matrix, while the Hessian matrix of deep neural networks is highly singular. To avoid the unidentifiability of deep neural networks in parameter estimation, we propose to conduct nonparametric estimation of partial derivatives with respect to inputs. We first show that model convergence of sparse deep neural networks is guaranteed in that the sample complexity only grows with the logarithm of the number of parameters or the input dimension when the $\ell_{1}$-norm of parameters is well constrained. Then by bounding the norm and the divergence of partial derivatives, we establish that the convergence rate of nonparametric estimation of partial derivatives scales as $\mathcal{O}(n^{-1/4})$, a rate which is slower than the model convergence rate $\mathcal{O}(n^{-1/2})$. To the best of our knowledge, this study combines nonparametric estimation and parametric sparse deep neural networks for the first time. As nonparametric estimation of partial derivatives is of great significance for nonlinear variable selection, the current results show the promising future for the interpretability of deep neural networks.
\end{abstract}



\end{frontmatter}

\section{Introduction}
Consider the regression problem $y_{i}=f_{0}(x_{i})+\xi_{i}$ given $n$ samples $(x_{i},y_{i})$, where the input predictive variable $x_{i}\in \mathbb{R}^d$ and response variable $y_{i}\in \R$, and $\xi_{i}$ represents random noise.
When the number of variables $d$ is greater than the sample size $n$, the relationship between response variables and predictive variables is not unique, which directly leads to the unidentifiability of statistical models. 

For linear regression when $f_{0}(\cdot)$ is a linear function, researchers have committed to utilize the sparse structure of high-dimensional data. Estimators such as Lasso \citep{RN298}, SCAD \citep{RN320}, MCP \citep{RN482} for sparse linear regression have been proposed with parameter estimation, prediction errors, and variable selection consistency established; see \cite{RN299} for a detailed review. 

However, in many practical situations such as genomic and financial data, the relationship between the covariate $x$ and the response $y$ is generally nonlinear, and modeling the nonlinear relationship become an important challenge for high-dimensional statistical learning. Subsequently, nonlinear regression models such as 
sparse additive model \citep{RN303, RN309, RN304}, local linear estimation \cite{RN306, RN307} and kernel regression \citep{RN310,RN311,RN278} have been developed. In view of the development trend, the expression ability of nonlinear models is getting more powerful, in order to achieve more effective approximation to the true nonlinear models.

Thanks to the powerful model expression ability of neural networks, more and more researchers have adopted neural networks in high-dimensional nonparametric statistical learning. From the aspect of generalization, in the pioneering study, \cite{RN380} showed that the sample complexity of two-layer neural networks with constrained $\ell_{1}$-norm on the weights only grows with the logarithm of the input dimension $d$. The dependence on the logarithm of the input dimensions $d$ is similar to the results in high-dimensional sparse linear regression and is crucial when the input dimension $d$ is large. Later, similar sample complexity that only depends on the logarithm of the input dimension $d$ has also been established for deep neural networks with $\ell_{1}$-constrained parameters \citep{RN503,RN488}. Other studies have established generalization guarantees for $\ell_{1}$-regularized deep neural networks \citep{RN333,RN323}, or some other kinds of sparsity such as node sparsity and layer sparsity \citep{RN321}. From the aspect of information-theoretic limitations, based on the combinatorial or hierarchical or piece-wise partitioning assumption of the true regression function, researchers showed that deep neural networks with sparse connection structure can achieve the minimax convergence rate to the true regression function without relying on the input dimension $d$ \citep{RN325,RN343,RN360}. 

Parameter estimation and variable selection have also been investigated.  \cite{RN70} imposed Lasso penalty on the first layer's weights of a two-layer convex neural network with infinite number of hidden nodes, theoretically ensuring the generalization and achieving nonlinear variable selection in high-dimensional settings. \cite{RN274} imposed sparse group Lasso penalty on the first layer's weights of two-layer neural networks, and showed that both the estimated model and weights of irrelevant features converge. \cite{RN369} extended the group Lasso penalty to group SCAD and MCP. \cite{RN340} also investigated the identifiability or estimation of parameters of two-layer neural networks with $\ell_{1}$-regularized weights based on an extended Restricted Isometry Property (RIP). In addition, \cite{RN341} showed that applying adaptive group Lasso penalty on the first layer's parameters of two-layer neural networks can achieve consistent model prediction, parameter estimation and variable selection. Consistency on parameter estimation and variable selection is based on the positive definite regularity of the Fisher information matrix or the Hessian matrix established in \cite{RN495}. Since the regularity condition of the Fisher information matrix or Hessian matrix is difficult to be satisfied for deep neural networks, \cite{RN338} extended the above results to deep neural networks via Lojasiewicz's inequality \citep{RN496}. However, in general, the constant in Lojasiewicz's inequality can be large for deep neural networks and how the constant depends on deep neural networks is not clear. Therefore, the convergence rate for parameter estimation via Lojasiewicz's inequality can be very slow for deep neural networks.

In summary, generalization theory for sparse deep neural networks is relatively better developed, while the parameter estimation and variable selection theories only focus on two-layer neural networks. The main challenge of parameter estimation and variable selection for deep neural networks lies in that deep neural networks are highly unidentifiable, since the Fisher information matrix or the Hessian matrix with respect to parameters can be highly singular \citep{Sagun2016,RN545}. Therefore, regularity conditions commonly used in classical high-dimensional sparse linear regression cannot be easily satisfied for deep neural networks, and classical parameter estimation or identification method that is important for variable selection  cannot be applied on sparse deep neural networks.

To avoid the unidentifiability of deep neural networks in parameter estimation, instead of focusing on parameter estimation, we propose to conduct nonparametric estimation of model partial derivatives with respect to inputs, which can be used to achieve nonlinear variable selection \citep{RN278}. 
In this study, we first establish the model convergence for $\ell_{1}$-norm constrained sparse deep neural networks; see Theorem \ref{model-convergence}. Specifically, the convergence rate of the model scales as $\mathcal{O}(\sqrt{\log(P)\E \| x \|_{\infty}^{2}} /n^{1/2})$, where $P$ is the number of parameters. This convergence rate is of great significance in high-dimensional settings, since both the terms $\log(P)$ and $\E \| x \|_{\infty}^{2}$ scale with the logarithm of input dimension $d$.
Secondly, we establish the convergence of partial derivatives by virtue of the efficiency of the $\ell_{1}$-norm constraint in bounding the norm and the divergence of partial derivatives; see Theorem \ref{convergence-gradient}. The convergence rate of partial derivatives scales as $\mathcal{O}(\sqrt{\log(P)\E \| x \|_{\infty}^{2}} /n^{1/4})$, which is slower than the convergence rate of the model by $n^{-1/4}$. Finally, numerical experiments are performed to show the performance of the nonparametric estimation. In a word, despite deep neural networks are highly unidentifiable with respect to parameters, our study shows that convergence of nonparametric estimation of partial derivatives can still be achieved under certain conditions. Since nonparametric estimation of partial derivatives is useful in identifying contributing variables,
our work on the convergence of partial derivatives is potentially significant for the interpretability of deep neural networks.  

\textbf{Notations}. We end this section by introducing some useful notations. For a vector $x=(x^{1},x^{2},\cdots,x^{d})\in \R^{d}$, define the $\ell_{p}$-norm of $x$ as $\|x\|_{p}=(\sum_{j=1}^{d}|x^{i}|^{p})^{1/p}$ for $1\leq p< \infty$ with $\|x\|_{\infty}=\max_{i=1,2,\cdots,d}|x^{i}|$. For two vectors $x,y\in \R^{d}$, we use $\cdot$ to denote their inner product, that is, $x\cdot y=y^{\top}x$. For a matrix $A\in \R^{d_{1}\times d_{2}}$, let $A_{ij}\ (i=1,\dots,d_{1},j=1,2,\cdots,d_{2})$ denote its $ij$-th entry, $A_{i\cdot}\ (i=1,\dots,d_{1})$ denote its $i$-th row, $A_{\cdot j}\ (j=1,2,\cdots,d_{2})$ denote its $j$-th column, and $\|A\|_{1,1}$ denote the $\ell_1$-norm along all the entries of $A$. For two matrices with the same dimension, we use $\odot$ to denote the element wise multiplication. For a function $f:\R^d\to \R$, $\nabla f$ and $\Delta f$ are used respectively to denote the gradient and the divergence, if they exist. For two measurable functions $f:\mathcal{X}\to \R^{d_{2}}$ and $g:\mathcal{X}\to \R^{d_{2}}$, where the independent variable $x$ is with probability density function $p(x)$ on a bounded sample space $\mathcal{X}\subseteq \R^{d}$, define the $L^{2}(p(x))$ (abbreviated as $L^{2}$) space (written as $f,g\in L^{2}(p(x))$), with its inner product and $L^{2}$-norm given respectively as
\begin{equation*}
\langle f,g \rangle _{L^{2}}=\int_{\mathcal{X}} f(x)\cdot g(x) p(x)dx   \quad \text{ and } \quad \| f \|_{L^{2}}^{2}=\langle f,f \rangle _{L^{2}},
\end{equation*}
and denote the $L^{\infty}$-norm as $\|g\|_{L^{\infty}} = \sup\{g(x): x \in \mathcal{X}\}$ when $d_{2}=1$.

\section{Problem setup}
In this section, we provide a detailed background on nonparametric regression problems and sparse deep neural networks used to perform estimation. The nonparametric estimation of partial derivatives with respect to inputs is also introduced with the suitable smooth activation function chosen.

\subsection{Nonparametric regression problems}
Consider the sparse nonparametric regression problem:
\begin{equation*}
y=f_{0}(x)+\xi
\end{equation*}
where the input prediction variable $x\in \mathbb{R}^d$, the response variable $y\in \R$, $\xi$ represents a zero-mean random noise, and $f_{0}$ is the true regression function to be estimated. The pair $(x,y)$ obeys the distribution with a joint probability density $p(x,y)$ on a bounded sample space $\mathcal{X}\times \mathcal{Y}$, and the marginal probability density of $x$ is $p_x(x)$. In order to perform estimation under the high-dimensional scenario, we assume that there are only $s$ sparse variables in $x$ responsible for predicting $y$. The relationship between $x$ and $y$ is then estimated via the hypothesis class of deep neural networks $\mathcal{F}_{\Theta}$ which will be specified later. We here assume that the unknown regression function $f_{0} \in \mathcal{F}_{\Theta}$, hence the learning model is well-specified and there exists no approximation error.
 
To estimate the unknown $f_{0}$ given a finite dataset $\mathbb{D}$ with $n$ independently and identically distributed samples $\{(x_{i},y_{i})\}_{i=1}^{n}$, a simple method is to minimize the empirical risk based on ordinary square loss $\hat{L}(f)=\frac1n\sum_{i=1}^n(y_i-f(x_i))^{2}$ with $f$ belonging to some certain function class. Denote the expected risk as $L(f)=\mathbb{E}_{(x,y)\sim p(x,y)}(y-f(x))^{2}$. The performance of an estimated model $\tilde{f}$ is then evaluated via the excess risk $L(\tilde{f})-L(f_{0})$, which equals to
\begin{equation*}
L(\hat{f})-L(f_0)=\E_{x\sim p_x(x)}(\hat{f}(x)-f_0(x))^2=\|\hat{f}-f_{0}\|^2_{L^{2}}.
\end{equation*}

\subsection{Sparse deep neural networks}
A deep neural network $f_\Theta: \R^d \rightarrow \R$ with $L$ layers can be represented as follows:
\begin{equation*}
f_\Theta(x)=\theta_L\sigma(\theta_{L-1}\cdots\sigma(\theta_1 x)\cdots)
\end{equation*}
where $\theta_l\in \R^{d_l\times d_{l-1}}(l=1,\cdots,L,\ d_0=d,\ d_L=1)$, $\Theta=\{\theta_L,\theta_{L-1},\cdots,\theta_1\}$, and $\sigma(\cdot)$ is the nonlinear activation function, which is assumed to be smooth and 1$\text{-}$Lipschitz and satisfy $\sigma(0)=0$. 

In the high-dimensional setting where the dimension $d$ or the number of parameters in $\Theta$ is larger than the number of samples $n$, the performance of $\hat{f}$ obtained through minimizing the empirical ordinary least squares risk is usually unsatisfactory and some sparse constraints on parameters are needed. For this purpose, denote $\|\cdot\|_{1,1}$ to be the $\ell_1$-norm along all the entries of a matrix $\theta_{l}$, that is,
\begin{equation} \label{matrix-1norm}
\| \theta_{l} \|_{1,1} = \sum_{k=1}^{d_{l}}\sum_{j=1}^{d_{l-1}} |(\theta_{l})_{kj}|.
\end{equation}
Then the $\ell_1$-norm of the parameter set $\Theta$ is defined as
\begin{equation} \label{parame-1norm}
\norm{\Theta}_{1} = \sum_{l=1}^{L} \| \theta_{l} \|_{1,1}.
\end{equation}
Denote $\varTheta_{r}=\{\Theta: \norm{\Theta}_{1} \leq r\}$, and the hypothesis class is denoted as $\mathcal{F}_{r}=\{f_{\Theta}:\Theta\in \varTheta_{r}\}$ by constraining the $\ell_{1}$-norm of the parameter set $\Theta$.
Then we propose to estimate the true regression function via minimizing the following objective function with sparse constraints

\begin{equation}\label{eq-mixed-sparse}
\hat{f}\in \argmin_{f_\Theta \in \mathcal{F}_{r}} \frac1n\sum\limits_{i=1}^n (y_i-f_\Theta(x_i))^{2}.
\end{equation}
In addition, the condition that $\norm{ \Theta_{0} }_{1}\leq r$ is required, where $\Theta_{0}$ is the corresponding  parameter set of $f_{0}$, in order to ensure the feasibility of $f_{0}$. In practice, the parameter $r$ can be chosen via parameter searching.

\subsection{Nonparametric estimation of partial derivatives}
We conduct nonparametric estimation of partial derivatives with respect to inputs. For a sample $x=(x^1,\cdots,x^d)^{\top}$ and the corresponding output of a deep neural network $f(x)$, the partial derivatives of $f(x)$ with respect to $x$, i.e.,
$\nabla_{x} f = ( \partial f(x)/\partial x^1,\cdots,\partial f(x)/\partial x^d )^{\top}$, is adopted to evaluate the importance of $x^a$ to the output. Specifically, the partial derivative characterizes the local variation of $f(x)$ when $x^a$ makes an infinitely small variation, and thus have been used for variable selection \citep{RN278}. In this study, we mainly focus on the convergence of nonparametric estimation of partial derivatives. The convergence of nonparametric estimation of partial derivatives is characterized via $\| \nabla_{x}\hat{f} - \nabla_{x}f_{0}  \|_{L^{2}}^{2}$. 

Recall the smoothness requirement on the nonlinear activation function, and the reason is to ensure the estimation of partial derivatives. Therefore, the commonly-used rectified linear unit (relu) activation function is not suitable for nonparametric estimation. Instead, the smoothed version $\text{softplus}(x)=\log(1+\exp(x))$ is adopted with the intercept term $\log2$ subtracted to ensure $\sigma(0)=0$. Finally we obtain the suitable activation function $\sigma(x)=\log(1+\exp(x))-\log2$.

\section{Main results}

In this section, we provide our main results on the convergence of the estimated model and the model derivatives.  

\subsection{Convergence of the model}\label{sec-model}
Some definitions and assumptions are needed first to facilitate the analysis.




\begin{definition}[Rademacher complexity]
Let $\epsilon_{1},\dots,\epsilon_{n}$ be $n$ independent Rademacher random variables that take values of $1$ or $-1$ with probability $1/2$. Given a dataset $S=\{x_{1},\dots,x_{n}\}$ with $n$ independent samples drawn from $p_x(x)$, 
the empirical Rademacher complexity is defined as
\begin{equation*}
\mathcal{R}_{S}(\mathcal{F}_{r}) = \E_{\epsilon}\left[\sup\limits_{f \in \mathcal{F}_{r}} \frac{1}{n}\sum\limits_{i=1}^n \epsilon_{i} f(x_{i})\right].
\end{equation*}
The Rademacher complexity of hypothesis class $\mathcal{F}_{r}$ is defined as

\begin{equation*}
\mathcal{R}_{n}(\mathcal{F}_{r}) =\E_{x}[\mathcal{R}_{S}(\mathcal{F}_{r})]= \E_{x}\E_{\epsilon}\left[\sup\limits_{f \in \mathcal{F}_{r}} \frac{1}{n}\sum\limits_{i=1}^n \epsilon_{i} f(x_{i})\right].
\end{equation*}
\end{definition}

\begin{definition}[Covering number]
The $\delta$-covering number of set $\mathcal{Q}$ with respect to metric $\rho$ is defined as the minimum size of $\delta$-cover $\mathcal{C}$ of $\mathcal{Q}$, such that for each $v \in \mathcal{Q}$, there exists $v' \in \mathcal{C}$ satisfying $\rho(v,v') \leq \delta$:
\begin{equation*}
\mathcal{N}(\delta, \mathcal{Q}, \rho) = \inf \{ |\mathcal{C}|: \mathcal{C} \text{ is a }  \delta \text{-cover of } \mathcal{Q} \text{ with respect to metric } \rho \}.
\end{equation*}
\end{definition}
In order to estimate the covering number, for the function class $\mathcal{F}_{r}$, we define the sample $L^{2}(P_{n})$-norm of a function $f \in \mathcal{F}_{r}$ and the derived metric respectively as
\begin{equation*}
\| f \|_{L^{2}(P_{n})} = \sqrt{ \frac{1}{n} \sum_{i=1}^{n} (f(x_{i}))^{2} }\quad \text{and}\quad \rho(f,f') = \| f-f' \|_{L^{2}(P_{n})}.
\end{equation*}
For the parameter set $\varTheta_{r}$, define the Frobenius-norm of parameter $\Theta \in \varTheta_{r}$ and the derived metric respectively as
\begin{equation*}
\norm{\Theta}_{F} = \sqrt{\sum_{l=1}^{L} \| \theta_{l} \|_{F}^{2}} = \sqrt{\sum_{l=1}^{L}\sum_{k=1}^{d_{l}}\sum_{j=1}^{d_{l-1}} (\theta_{l})_{kj}^{2}}\quad \text{and}\quad  \rho(\Theta,\Theta') = \norm{ \Theta-\Theta' }_{F}. 
\end{equation*}
We shall see in the proof that the covering number of the function space $\mathcal{F}_{r}$ can be bounded by the covering number of the parameter space $\varTheta_{r}$.

\begin{assumption}[Bounded loss function]\label{bounded-assumption}
For any loss function $f_{\Theta} \in \mathcal{F}_{r}$, $f_{\Theta}$ is assumed to be bounded, such that $| f_{\Theta}(x)-y | \leq b_{0}$ for all $(x,y) \in \mathcal{X}\times \mathcal{Y}$, where $b_{0}$ is an absolute positive constant.
\end{assumption}

\begin{assumption}[Bounded inputs]\label{bounded-inputs}
The input space  $\mathcal{X}$ is assumed to be bounded such that the $\ell_{\infty}$-norm of each $x \in \mathcal{X}$ is within $R$, i.e., $\sup \{\| x \|_{\infty}: x \in \mathcal{X}\}\leq R$.
\end{assumption}

Then the following lemma is an oracle inequality that provides an upper bound of the expected $L^{2}$ error of the sparse constraint estimated model $\hat{f}$. Recall that the dataset $\mathbb{D}$ stands for $n$ independently and identically distributed samples $\{(x_{i},y_{i})\}_{i=1}^{n}$.
\begin{lemma}[Oracle inequality]\label{oracle}
Let $\hat{f}$ be the estimated model obtained from \eqref{eq-mixed-sparse}. Then under Assumption \ref{bounded-assumption}, it follows that

\begin{equation*}
\mathbb{E}_{\mathbb{D}} \| \hat{f}-f_{0} \|_{L^{2}}^{2}  \leq 4b_{0}\mathcal{R}_{n}(\mathcal{F}_{r}).
\end{equation*}
\end{lemma}

Note from Lemma \ref{oracle} that the convergence of the model is determined by the Rademacher complexity. We next bound the Rademacher complexity of the hypothesis class via the covering number. The key is to utilize the following  Lipschitz properties of the model function. 

\begin{lemma}[Lipschitz property with respect to parameters]\label{lipschitz-parameter}
Assume that the activation function $\sigma(\cdot)$ is 1-Lipschitz and satisfy $\sigma(0)=0$. For each $x \in \mathcal{X}$ and $\Theta, \Theta' \in \varTheta_{r}$, it follows that
\begin{equation*}
\begin{aligned}
 | f_{\Theta}(x) - f_{\Theta'}(x) | &\leq \emph{Lip}(x) \norm{\Theta-\Theta'}_{F},\\
 \| f_{\Theta}(x) - f_{\Theta'}(x) \|_{L^{2}(P_{n})} &\leq \emph{Lip}_{L^{2}(P_{n})}(x) \norm{\Theta-\Theta'}_{F}, 
\end{aligned}
\end{equation*}
where
\begin{align}
\emph{Lip}(x) &= \sqrt{L} \left( \frac{r}{L-1} \right)^{L-1} \| x \|_{\infty},\notag \\
\emph{Lip}_{L^{2}(P_{n})}(x) &= \sqrt{L} \left( \frac{r}{L-1} \right)^{L-1} \sqrt{ \frac{1}{n} \sum_{i=1}^{n} \| x_{i} \|_{\infty}^{2} }.\label{eq-LipL2}
\end{align}
\end{lemma}

Then the Rademacher complexity of $\mathcal{F}_{r}$ is provided by transforming the covering number of the function space to the covering number of the parameter space via the Lipschitz property. Denote the total number of parameters as $P$.

\begin{lemma}[Rademacher complexity of sparse deep neural networks]\label{rademacher}
Assume $f_{\Theta} \in \mathcal{F}_{r}$ is \emph{Lip}$(x)$-Lipschitz with respect to parameters. Then the Rademacher complexity of $\mathcal{F}_{r}$ follows that
\begin{equation*}
\begin{aligned}
\mathcal{R}_{S}(\mathcal{F}_{r}) &\leq 24r \left( \frac{r}{L-1} \right)^{L-1} \sqrt{\frac{2L\log P}{n}} \left(1+\log(c_{1}\sqrt{n}) \sqrt{ \frac{1}{n} \sum_{i=1}^{n} \| x_{i} \|_{\infty}^{2} }  \right), \\
\mathcal{R}_{n}(\mathcal{F}_{r}) &\leq 24r \left( \frac{r}{L-1} \right)^{L-1} \sqrt{\frac{2L\log P}{n}} \left(1+\log(c_{1}\sqrt{n}) \sqrt{\E \| x \|_{\infty}^{2} }  \right), \\
\end{aligned}
\end{equation*}
where
\begin{equation*}
c_{1} = \frac{R}{6rL^{\frac{3}{2}}\sqrt{2\log P}}.
\end{equation*}
\end{lemma}

Combining Lemmas \ref{oracle} and \ref{rademacher}, it is straightforward to arrive at the model convergence as follows with the proof omitted.
\begin{theorem}[Convergence of the model]\label{model-convergence}
Let $\hat{f}$ be the estimated model obtained from \eqref{eq-mixed-sparse}. Then under Assumptions \ref{bounded-assumption} and \ref{bounded-inputs}, it follows that
\begin{equation*}
\mathbb{E}_{\mathbb{D}} \| \hat{f}-f_{0} \|_{L^{2}}^{2}  \leq 96b_{0}r \left( \frac{r}{L-1} \right)^{L-1} \sqrt{\frac{2L\log P}{n}} \left(1+\log(c_{1}\sqrt{n}) \sqrt{\E \| x \|_{\infty}^{2} }  \right),
\end{equation*}
where
\begin{equation*}
c_{1} = \frac{R}{6rL^{\frac{3}{2}}\sqrt{2\log P}}.
\end{equation*}
\end{theorem}
Theorem \ref{model-convergence} provides an upper bound for the estimated model $\hat{f}$ via sparse deep neural networks and the true nonparametric regression function $f_{0}$. Note that the convergence of the model is captured essentially by the Rademacher complexity, and the upper bound of model convergence constitutes three important parts as follows.
\begin{remark}
\rm{(i)} The pre-factor $(r/(L-1))^{L-1}$ may increase exponentially with the number of layers $L$ when the average parameter norm $(r/(L-1))$ is larger than 1. This exponential dependence on the number of layers is due to the worst case Lipschitz property of deep neural networks and is manifested in results concerning the sample complexities of deep neural networks \citep{RN503,RN488,RN323,RN481}.

\rm{(ii)} The pre-factor $\sqrt{\log P}$ is due to the covering number of the $\ell_{1}$-norm constrained parameter space. Particularly, assume that all hidden layers have the same number of hidden units $h$ which is much smaller than the input dimension $d$. Then the total number of parameters $P=d*h+(L-2)h^{2}+h$ and the pre-factor $\sqrt{\log P}$ scales linearly with $\sqrt{\log d}$. This logarithm dependence on the number of parameters or the input dimension is especially important under the high-dimensional scenario.

\rm{(iii)} \cite{RN323} also use the covering number technique to bound the excess risk or $L^{2}$ error of $\ell_{1}$-norm constrained sparse deep neural networks. However, the term $\sqrt{\E \|x\|^{2}_{\infty}}$ in Theorem \ref{model-convergence} is better than the term $\sqrt{\sum_{i=1}^{n} \|x_{i}\|^{2}_{2}/n}$ or $\sqrt{\E \|x\|^{2}_{2}}$ in \cite{RN323} when the input dimension $d$ is high, since the term $\sqrt{\E \|x\|^{2}_{\infty}}$ scales with $\sqrt{\log d}$ but the term $\sqrt{\E \|x\|^{2}_{2}}$ scales with $\sqrt{d}$.
\end{remark}

\subsection{Convergence of partial derivatives}\label{sec-derivative}

Convergence of model is prerequisite but not sufficient for the convergence of partial derivatives. In the following, we introduce several  lemmas in preparation for establishing the convergence of partial derivatives.

\begin{lemma}[Bounded norm of partial derivatives] \label{lipschitz-input}
Assume the softplus activation function $\sigma(\cdot)$ is adopted. For any $x \in \mathcal{X}$ and $\Theta \in \varTheta_{r}$, it follows that
\begin{equation*}
\| \nabla_{x}f_{\Theta}(x) \|_{1} \leq  (r/L)^{L}.
\end{equation*}
\end{lemma}

\begin{lemma}[Bounded divergence of partial derivatives]\label{bounded-divergence}
Assume the softplus activation function $\sigma(\cdot)$ is adopted. For any $x \in \mathcal{X}$ and $\Theta \in \varTheta_{r}$, it follows that
\begin{equation*}
| \Delta_{x}f_{\Theta}(x) | \leq \frac{L}{4}\left(\frac{r}{L}\right)^{L} \max_{k\in\{2,\cdots,L-1\}} \left( \frac{r}{k} \right)^{k}.
\end{equation*}
\end{lemma}

\begin{remark}\label{remark-bounded-derivatives}
\rm{(i)} It is well known in real analysis that function convergence does not guarantee derivative convergence. For example, let $f_{n}(x)=n^{-k}\sin(nx)$ and $f_{0}(x)=0$, $k \in (0,1)$, then $f_{n}$ converges to $f_{0}$ as $n \rightarrow \infty$, but $\nabla_{x}f_{n}$ does not converges to $\nabla_{x}f_{0}$. This pathological example is due to the fact that the first or second derivative is not bounded. Therefore, as we will show in subsequent results, the requirements on bounded norm of partial derivatives (cf. Lemma \ref{lipschitz-input}) and bounded divergence of partial derivatives (cf. Lemma \ref{bounded-divergence}) are crucial for the convergence of partial derivatives.

\rm{(ii)} In the proof of Lemma \ref{bounded-divergence}, we use a key fact of the softplus activation function that the second derivatives of hidden layers are well bounded, that is, $h''_{k} = \sigma''(\cdots) \in (0,1/4)$. The bounded second derivative does not hold for the relu activation function, since $relu(\cdot)$ is not smooth and the second derivative can be infinite. Therefore, the proof of Lemma \ref{bounded-divergence} implies that relu is not suitable for derivative estimation and adopting a smooth activation function such as softplus is necessary. 
\end{remark}
\begin{assumption}[Boundary condition]\label{boundary-assumption}
The true regression function $f_{0}$ and the estimated function $\hat{f}$ are both assumed to have zero normal derivative on the boundary, that is, $\nabla_{x}f_{0}\cdot\vec{n}=0$ and $\nabla_{x}\hat{f}\cdot\vec{n}=0$, where $\vec{n}$ is the unit normal to the boundary.
\end{assumption}
Since it is not possible to estimate the partial derivatives on the boundary $\partial \mathcal{X}$ out of where there are no samples, we need the boundary assumption.
\begin{assumption}[Bounded derivatives of probability density]\label{bounded-density}
The $L^{\infty}$-norm of the logarithm of probability density is assumed to be bounded, that is, $\sup \{ |\nabla_{x_{i}} \log p_x(x)|: x\in \mathcal{X}, \forall i\in \{1,\cdots,d\} \} \leq b_{1}$, where $b_{1}$ is an absolute positive constant.
\end{assumption}

It is worth noting that Assumption \ref{bounded-density} is reasonable. For instance, assume that $x$ obeys a truncated normal distribution such that $\mathcal{X}$ is bounded. When all the elements of $x$ are independent, it holds that $|\nabla_{x_{i}} \log p_x(x)|=|x_{i}-\mu_{i}|/\sigma_{i}^{2}$, and thus $\sup |\nabla_{x_{i}} \log p_x(x)|$ is bounded provided that $x \in \mathcal{X}$ is bounded (i.e., Assumption \ref{bounded-inputs} is satisfied).

\begin{lemma}[Green's formula \citep{RN357}]\label{green-lemma}
Let $f,g \in L^{2}(p_x(x))$ and $\nabla_{x}f,\nabla_{x}g \in L^{2}(p_x(x))$ with $\nabla_{x}f\cdot\vec{n}=0$ on the boundary $\partial \mathcal{X}$, where $\vec{n}$ is the unit normal to the boundary. Then it follows that
\begin{equation*}
-\langle \nabla_{x}f,\nabla_{x}g \rangle_{L^{2}} = \langle \Delta_{x}f+\nabla_{x}f \cdot \nabla_{x} \log p_x(x), g \rangle_{L^{2}}.
\end{equation*}
\end{lemma}

With these Lemmas, we can now establish the other main result on the convergence of model partial derivatives with respect to inputs.
\begin{theorem}[Convergence of partial derivatives]\label{convergence-gradient}
Let $\hat{f}$ be the estimated model obtained from \eqref{eq-mixed-sparse}. Then under Assumption \ref{bounded-assumption}, \ref{bounded-inputs}, \ref{boundary-assumption} and \ref{bounded-density}, it follows that
\begin{equation*}
\begin{aligned}
&\E_{\D} \|\nabla_{x}\hat{f} - \nabla_{x}f_{0}\|_{L^{2}}^{2} \leq \frac{1}{n^{1/4}}\left(\frac{r}{L}\right)^{2L}\left(2+ \frac{L^{2}}{8} \max_{k\in\{2,\cdots,L-1\}} \left( \frac{r}{k} \right)^{2k}\right) \\
&+  48(1+b_{1})b_{0}r \left( \frac{r}{L-1} \right)^{L-1} \frac{\sqrt{2L\log P}}{n^{1/4}} \left(1+\log(c_{1}\sqrt{n}) \sqrt{\E \| x \|_{\infty}^{2} }  \right),
\end{aligned}
\end{equation*}
where
\begin{equation*}
c_{1} = \frac{R}{6rL^{\frac{3}{2}}\sqrt{2\log P}}.
\end{equation*}
\end{theorem}

\begin{remark}
The convergence rate of partial derivatives constitutes of two parts. The first part is due to the bounded norm of the gradient and the bounded divergence of the gradient field, both of which are established in Lemma \ref{lipschitz-input} and Lemma \ref{bounded-divergence}, respectively. The second part comes from the convergence of model that is established in Theorem \ref{model-convergence}. In a word, Theorem \ref{convergence-gradient} shows that in order to guarantee the convergence of partial derivatives, it is not sufficient to rely on the model convergence alone, but the boundedness of the gradient and the divergence of the gradient field are also required. This result coincides with the discussion in Remark \ref{remark-bounded-derivatives}. In addition, it should be noted that the convergence rate of partial derivatives is slower than the convergence rate of model by $n^{1/4}$.
\end{remark}

\section{Experiments}
In this section, we perform experiments to illustrate the theoretical results. We demonstrate that $\ell_{1}$-norm constrained sparse deep neural networks is efficient in nonparametric estimation under high-dimensional scenarios, and the adoption of smooth activation functions is crucial for nonparametric estimation.

In the well-specified case, the true regression function is assumed to be a neural network $f_0(x)=\theta_L\sigma(\theta_{L-1}\cdots\sigma(\theta_1 x)\cdots)$ and the data are generated via $y_{i}=f_{0}(x_{i})+\xi_{i}$. Each layer's parameters of $f_{0}$ are generated from a normal distribution with standard deviation being $\sqrt{2/h}$, where $h$ represents the number of input features of each layer. We set the number of hidden units of each layer as 10. The input prediction variable $x$ and random noise $\xi$ are generated from truncated normal distributions with standard deviation 1 and 0.1 respectively. We truncate a normal distribution at a large absolute value (10 times the standard deviation) and uniformly scaling the density values inside the bounded interval.
For the high-dimensional input prediction variable $x \in \R^{d}$, we assume only $s$ sparse variables are responsible for the regression function $f_0(x)$. We set $d=100$ and $s=5$ across all experiments. We just keep the first layer's parameters $\theta_{1}$ to be nonzeros corresponding to the $s$ relevant sparse variables and set the other parameters corresponding to irrelevant variables as zeros.

The number of training samples varies from 50 to 100, and the number of testing samples is set as 10000. We report the $L^{2}$ error of testing samples for model prediction and derivative estimation. We use a large testing sample size to approximate the expected $L^{2}$ error. The training process is repeated with random training samples for 100 times and results are reported via averaging.



\begin{figure}[htbp]
\includegraphics[width=\textwidth]{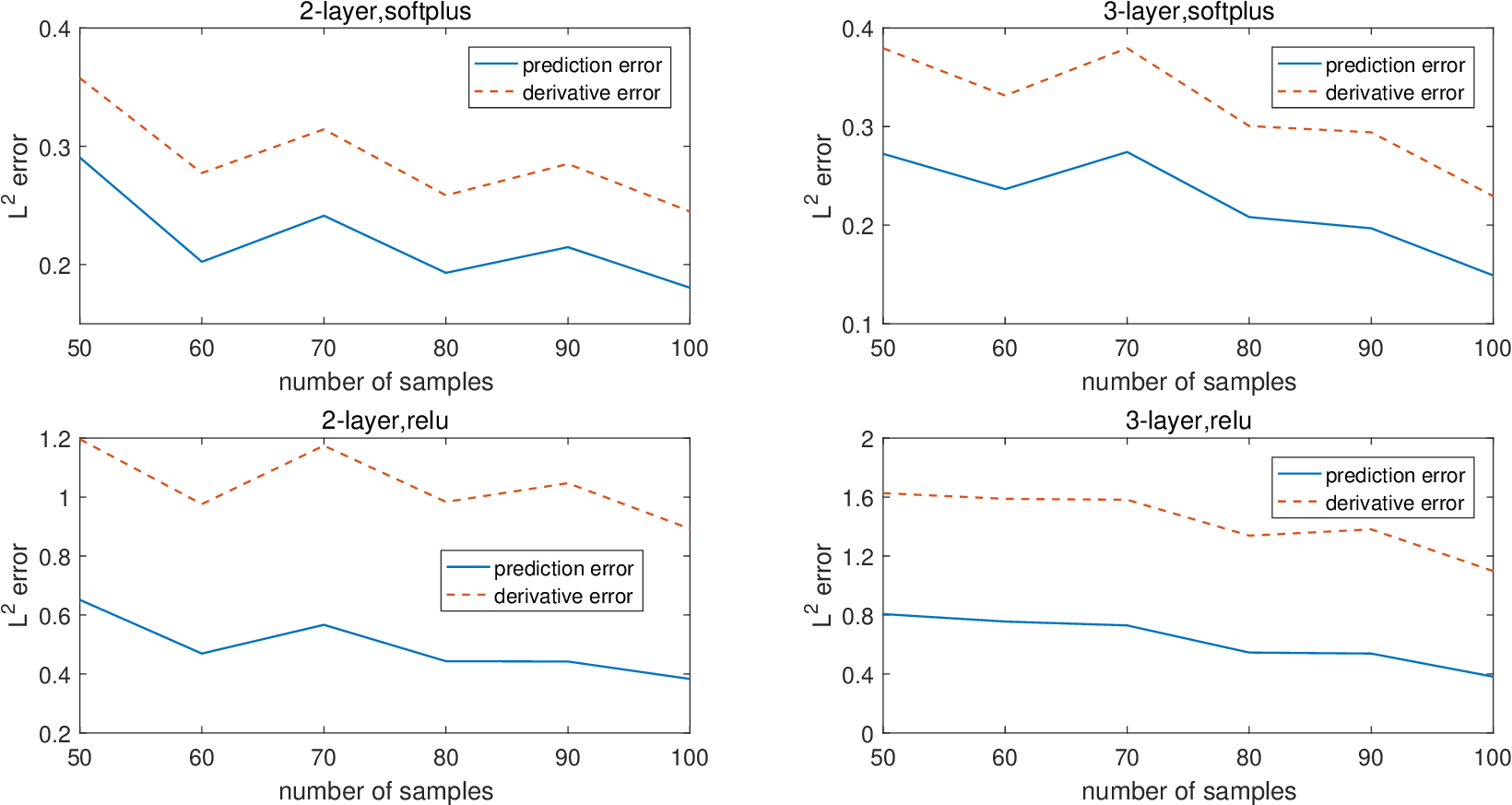}
\caption{$L^{2}$ error of model prediction and derivative estimation.}
\label{fig1}
\end{figure}

We compare the results of two nonlinear activation function, i.e., the $\text{relu}(x)=\max\{0,x\}$ and the smooth $\text{softplus}(x)=\log(1+\exp(x))-\log2$. For the 2-layer and 3-layer neural networks, results on $L^{2}$ errors of model prediction and derivative estimation are displayed in Fig. \ref{fig1}. As we can see from Fig. \ref{fig1}, both the prediction error and the derivative estimation error decrease as the sample size increases. However, the difference between the prediction error and the derivative estimation error for the softplus function is much smaller than that of the relu function, a phenomenon which indicates the significance of the smooth function in nonparametric estimation of the model derivatives.


\section{Conclusion}
In this study, we demonstrate that despite the unidentifiability of deep neural networks in parameter estimation, one can still achieve convergence of nonparametric estimation of partial derivatives with respect to inputs for sparse deep neural networks. The established convergence of nonparametric estimation is of potential significance for the interpretability of deep neural networks.

There are also some future works to be considered. First, the theoretical results are based on the bounded domain and the boundary assumptions. It would be interesting to weaken these assumptions to build convergence results. Second, the theoretically established convergence rate of partial derivatives is slower than that of model by $n^{-1/4}$. It is natural to ask under what conditions can faster convergence rate of partial derivatives be guaranteed. Thirdly, one can also utilize the established convergence of partial derivatives to achieve the consistency of nonparametric variable selection.
 

\begin{appendix}

\section{Proof of Section \ref{sec-model}}\label{sec-appA}

\begin{proof}[Proof of Lemma \ref{oracle}]
It follows that the excess risk equals to the $L^{2}$ error
\begin{equation*}
\begin{aligned}
L(\hat{f})-L(f_{0}) &= \E_{(x,y) \sim p(x,y)}(\hat{f}(x)-y)^{2} - \E_{(x,y) \sim p(x,y)}(f_{0}(x)-y)^{2} \\
&=\E_{(x,y) \sim p(x,y)}(\hat{f}(x)-f_{0}(x)-\xi)^{2} - \E_{(x,y) \sim p(x,y)}(f_{0}(x)-f_{0}(x)-\xi)^{2} \\
&=\E_{x \sim p_x(x)}(\hat{f}(x)-f_{0}(x))^{2} - \E_{(x,y) \sim p(x,y)}[(\hat{f}(x)-f_{0}(x))\xi] \\
&=\E_{x \sim p_x(x)}(\hat{f}(x)-f_{0}(x))^{2} \\
&=\| \hat{f} - f_{0} \|_{L^{2}}^{2}.
\end{aligned}
\end{equation*}
Decompose the $L^{2}$ error as follows:
\begin{equation}\label{decompose}
\begin{aligned}
\| \hat{f} - f_{0} \|_{L^{2}}^{2} &= L(\hat{f})-L(f_{0}) \\
&=[L(\hat{f}) - \hat{L}(\hat{f})] + [\hat{L}(\hat{f})-\hat{L}(f_{0})] + [\hat{L}(f_{0})-L(f_{0})].
\end{aligned}
\end{equation}
Since $\norm{ \Theta_{0} }_{1}\leq r$, it follows from the global optimality of $\hat{f}$ that
\begin{equation}\label{optimality}
\hat{L}(\hat{f})  \leq \hat{L}(f_{0}).
\end{equation}
Combining \eqref{decompose} and \eqref{optimality} yields that
\begin{equation*}
\begin{aligned}
&\| \hat{f} - f_{0} \|_{L^{2}}^{2} \leq [L(\hat{f}) - \hat{L}(\hat{f})] + [\hat{L}(f_{0})-L(f_{0})].
\end{aligned}
\end{equation*}
Taking expectations with respect to the dataset $\D$ and noting the fact that $\E_{\D}\hat{L}(f_{0})=L(f_{0})$, we obtain that
\begin{equation*}
\begin{aligned}
\E_{\D}\| \hat{f} - f_{0} \|_{L^{2}}^{2} \leq \E_{\D}[L(\hat{f}) - \hat{L}(\hat{f})].
\end{aligned}
\end{equation*}
Set $\mathcal{L}_{r}=\{(x,y) \rightarrow \ell(f(x),y):f \in \mathcal{F}_{r}\}$ to be the family of loss functions. The term $\E_{\D}[L(\hat{f}) - \hat{L}(\hat{f})]$ is bounded via Rademacher complexity as follows
\begin{equation*}
\begin{aligned}
&\E_{\D}[L(\hat{f}) - \hat{L}(\hat{f})] \leq \E_{\D}\left[\sup\limits_{f \in \mathcal{F}_{r}}(L(f) - \hat{L}(f))\right] \\
&= \E_{\D}\left[\sup\limits_{\ell \in \mathcal{L}_{r}}\left(\E(\ell) -\frac{1}{n}\sum_{i=1}^{n}\ell(f(x_{i}),y_{i})\right)\right] \\
&=  \E_{\D}\left[\sup\limits_{\ell \in \mathcal{L}_{r}} \left(\E_{\D'} \left[\frac{1}{n}\sum_{i=1}^{n}\ell(f(x_{i}'),y_{i}')  -\frac{1}{n}\sum_{i=1}^{n}\ell(f(x_{i}),y_{i})\right]\right)\right] \\
&\leq \E_{\D}\E_{\D'}\left[ \sup\limits_{\ell \in \mathcal{L}_{r}} \left( \frac{1}{n}\sum_{i=1}^{n}\ell(f(x_{i}'),y_{i}')  -\frac{1}{n}\sum_{i=1}^{n}\ell(f(x_{i}),y_{i})\right)\right] \\
&= \E_{\D}\E_{\D'}\E_{\epsilon}\left[ \sup\limits_{\ell \in \mathcal{L}_{r}} \left( \frac{1}{n}\sum_{i=1}^{n} \epsilon_{i} \ell(f(x_{i}'),y_{i}')  -\frac{1}{n}\sum_{i=1}^{n} \epsilon_{i} \ell(f(x_{i}),y_{i})\right)\right] \\
&\leq \E_{\D}\E_{\D'}\E_{\epsilon}\left[ \sup\limits_{\ell \in \mathcal{L}_{r}} \left( \frac{1}{n}\sum_{i=1}^{n} \epsilon_{i} \ell(f(x_{i}'),y_{i}') \right) +  \sup\limits_{\ell \in \mathcal{L}_{r}} \left(\frac{1}{n}\sum_{i=1}^{n} -\epsilon_{i} \ell(f(x_{i}),y_{i})\right)\right] \\
&= 2\E_{\D}\E_{\epsilon}\left[ \sup\limits_{\ell \in \mathcal{L}_{r}} \left( \frac{1}{n}\sum_{i=1}^{n} \epsilon_{i} \ell(f(x_{i}),y_{i}) \right) \right] \\
&=2\mathcal{R}_{n}(\mathcal{L}_{r}).
\end{aligned}
\end{equation*}
Since the loss function $\ell(f(x),y)=((f(x)-y))^{2}$ is $2b_{0}$-Lipschitz by Assumption \ref{bounded-assumption}, we have by the contraction inequality of Rademacher complexity of Lipschitz loss functions that
\begin{equation*}
\E_{\D}\| \hat{f} - f_{0} \|_{L^{2}}^{2} \leq \E_{\D}[L(\hat{f}) - \hat{L}(\hat{f})] \leq 2\mathcal{R}_{n}(\mathcal{L}_{r}) \leq 4b_{0}\mathcal{R}_{n}(\mathcal{F}_{r}).
\end{equation*}
The proof is complete.
\end{proof}

\begin{proof}[Proof of Lemma \ref{lipschitz-parameter}]
Denote
\begin{equation*}
\nabla_{\theta}f_{\Theta}(x) = \left(  \frac{\partial f}{\partial \theta_{1}}\Big|_{\text{vec}}^{\top}, \cdots,  \frac{\partial f}{\partial \theta_{L}}\Big|_{\text{vec}}^{\top}       \right)^{\top},
\end{equation*}
where for $i=1,2,\cdots,L$, $\frac{\partial f}{\partial \theta_{i}}\Big|_{\text{vec}}$ is the vectorized form of the matrix $\frac{\partial f}{\partial \theta_{i}}$.  
Since the Lipschitz constant equals to $ \sup\{ \| \nabla_{\theta}f_{\Theta}(x) \|_{2}: \Theta \in \varTheta_{r}  \}$, it suffice to analyze the $\ell_{2}$-norm of $\nabla_{\theta}f_{\Theta}(x)$ as follows
\begin{equation*}
\| \nabla_{\theta}f_{\Theta}(x) \|_{2} = \sqrt{ \sum_{l=1}^{L} \| \frac{\partial f}{\partial \theta_{l}} \|_{F}^{2}  }.
\end{equation*}
For a deep neural network $f_\Theta(x)=\theta_L\sigma(\theta_{L-1}\cdots\sigma(\theta_1 x)\cdots)$, denote the activations of the $l$-th hidden layer as $h_{l}=\sigma(\theta_{l}\cdots\sigma(\theta_1 x)\cdots)$ with $h_{0}(x)=x$. Then the derivatives of the $l$-th hidden layer is $h'_{l}=\sigma'(\theta_{l}\cdots\sigma(\theta_1 x)\cdots)$, where $\sigma'(\cdot)$ is the derivative of the activation function. According to the chain rule, the gradients with respect to parameters equal to 
\begin{equation*}
\frac{\partial f}{\partial \theta_{l}} =(h'_{l} \odot \theta_{l+1}^{\top} \cdots \theta_{L-1}^{\top} h'_{L-1} \odot \theta_{L}^{\top})  h_{l-1}^{\top}.
\end{equation*}
When the activation function is the softplus function $\sigma(x)=\log(1+\exp(x))-\log2$, the derivative $\sigma'(\cdot)$ is the sigmoid function $\sigma'(x)=\frac{\exp(x)}{1+\exp(x)}$, thus $\sigma'(\cdot) \in (0,1)$ and $h'_{l} \in (0,1)$. Recall the  $\| \cdot \|_{1,1}$ matrix norm in \eqref{matrix-1norm}, and the Frobenius-norm of gradients is bounded layer by layer as:
\begin{equation}\label{lip-model1}
\begin{aligned}
\| \frac{\partial f}{\partial \theta_{l}} \|_{F} &= \sqrt{\text{trace}(\frac{\partial f}{\partial \theta_{l}}  \frac{\partial f}{\partial \theta_{l}}^{\top})} \\
&= \| h'_{l} \odot \theta_{l+1}^{\top} \cdots \theta_{L-1}^{\top} h'_{L-1} \odot \theta_{L}^{\top} \|_{2}  \| h_{l-1} \|_{2} \\
&\leq \| \theta_{l+1}^{\top}  h'_{l+1} \cdots \theta_{L-1}^{\top} h'_{L-1} \odot \theta_{L}^{\top} \|_{2}  \| h_{l-1} \|_{2}  \quad (\because h'_{l} \in (0,1))\\
&\leq \| \theta_{l+1} \|_{1,1}  \| h'_{l+1} \cdots \theta_{L-1}^{\top} h'_{L-1} \odot \theta_{L}^{\top} \|_{\infty}  \| h_{l-1} \|_{2} \\
&\leq \| \theta_{l+1} \|_{1,1}  \| h'_{l+1} \cdots \theta_{L-1}^{\top} h'_{L-1} \odot \theta_{L}^{\top} \|_{2}  \| h_{l-1} \|_{2}  \quad (\because \|\cdot\|_{\infty} \leq \|\cdot\|_{2} ) \\
&\leq \| h_{l-1} \|_{2} \prod_{k=l+1}^{L-1} \| \theta_{k} \|_{1,1} \|h'_{L-1} \odot \theta_{L}^{\top}\|_{\infty} \\
&\leq \| h_{l-1} \|_{2} \prod_{k=l+1}^{L-1} \| \theta_{k} \|_{1,1} \| \theta_{L}\|_{\infty} \leq \| h_{l-1} \|_{2} \prod_{k=l+1}^{L-1} \| \theta_{k} \|_{1,1} \| \theta_{L}\|_{1} \\
& = \| h_{l-1} \|_{2} \prod_{k=l+1}^{L} \| \theta_{k} \|_{1,1},
\end{aligned}
\end{equation}
where the second inequality follows from H$\ddot{o}$lder's inequality that $\|Ax\|_{2} = \sqrt{ \sum_{i} \langle A_{i\cdot},x \rangle ^{2} } \leq \sqrt{ \sum_{i} \|A_{i\cdot}\|_{1}^{2} \|x \|_{\infty}^{2} } \leq \sqrt{ (\sum_{i} \|A_{i\cdot}\|_{1})^{2} \|x \|_{\infty}^{2} } = \|A\|_{1,1}\|x \|_{\infty} $. Furthermore, the term $\| h_{l-1} \|_{2}$ is bounded layer by layer as:
\begin{equation}\label{lip-model2}
\begin{aligned}
\| h_{l-1} \|_{2} &= \| \sigma(\theta_{l-1}\cdots\sigma(\theta_1 x)\cdots) \|_{2} \\
&\leq \| \theta_{l-1} \sigma(\theta_{l-2} \cdots\sigma(\theta_1 x)\cdots) \|_{2} \\
&\leq \| \theta_{l-1} \|_{1,1}  \| \sigma(\theta_{l-2} \cdots\sigma(\theta_1 x)\cdots) \|_{\infty} \\
&\leq \| \theta_{l-1} \|_{1,1}  \| \sigma(\theta_{l-2} \cdots\sigma(\theta_1 x)\cdots) \|_{2} \\
&\leq \| x \|_{\infty} \prod_{k=1}^{l-1} \| \theta_{k} \|_{1,1},
\end{aligned}
\end{equation}
where the first inequality is due to the 1-Lipschitz property of $\sigma(\cdot)$ and fact that $\sigma(0)=0$, and the last inequality is by iteration. Combining \eqref{lip-model1} and \eqref{lip-model2} yields that
\begin{equation*}
\| \frac{\partial f}{\partial \theta_{l}} \|_{F} \leq \| x \|_{\infty} \prod_{k=1,k \neq l}^{L} \| \theta_{k} \|_{1,1}.
\end{equation*}
Combing the Frobenius-norm of gradients of all layers, we obtain that 
\begin{equation*}
\begin{aligned}
\| \nabla_{\theta}f_{\Theta}(x) \|_{2} &= \sqrt{ \sum_{l=1}^{L} \| \frac{\partial f}{\partial \theta_{l}} \|_{F}^{2}  } \leq \| x \|_{\infty} \sqrt{ \sum_{l=1}^{L} \prod_{k=1,k \neq l}^{L} \| \theta_{k} \|_{1,1}^{2}} \\
&\leq \| x \|_{\infty} \sqrt{ L \max_{l \in \{1,\dots,L \}} \prod_{k=1,k \neq l}^{L} \| \theta_{k} \|_{1,1}^{2}}   \\
&=  \sqrt{L} \| x \|_{\infty} \max_{l \in \{1,\dots,L \}} \prod_{k=1,k \neq l}^{L} \| \theta_{k} \|_{1,1} \\
&\leq \sqrt{L} \| x \|_{\infty} \max_{l \in \{1,\dots,L \}} \left( \frac{1}{L-1} \sum_{k=1,k \neq l}^{L} \| \theta_{k} \|_{1,1} \right)^{L-1} \\
&\leq \sqrt{L} \| x \|_{\infty} \left( \frac{1}{L-1} \sum_{k=1}^{L} \| \theta_{k} \|_{1,1} \right)^{L-1} \\
&\leq \sqrt{L} \left( \frac{r}{L-1} \right)^{L-1} \| x \|_{\infty}.\\
\end{aligned}
\end{equation*}
Furthermore, we can bound Lip$_{L^{2}(P_{n})}(x)$ as:
\begin{equation*}
\begin{aligned}
\| f_{\Theta}(x) - f_{\Theta'}(x) \|_{L^{2}(P_{n})} &= \sqrt{ \frac{1}{n} \sum_{i=1}^{n} \left(f_{\Theta}(x_{i}) - f_{\Theta'}(x_{i}) \right)^{2} } \\
&\leq \sqrt{ \frac{1}{n} \sum_{i=1}^{n} \left(\text{Lip}_{\theta,f}(x_{i}) \norm{\Theta-\Theta'}_{F} \right)^{2} } \\
&\leq \sqrt{ \frac{1}{n} \sum_{i=1}^{n} \left( \sqrt{L} \left( \frac{r}{L-1} \right)^{L-1} \| x_{i} \|_{\infty} \norm{\Theta-\Theta'}_{F} \right)^{2} } \\
&= \sqrt{L} \left( \frac{r}{L-1} \right)^{L-1} \sqrt{ \frac{1}{n} \sum_{i=1}^{n} \| x_{i} \|_{\infty}^{2} }  \ \  \norm{\Theta-\Theta'}_{F}.
\end{aligned}
\end{equation*}
The proof is complete.
\end{proof}

\begin{proof}[Proof of Lemma \ref{rademacher}]
The empirical Rademacher complexity is bounded via Dudley's integral as follows
\begin{equation}\label{dudley}
\mathcal{R}_{S}(\mathcal{F}_{r}) \leq 4\alpha + 12\int_{\alpha}^{\infty} \sqrt{ \frac{\log \mathcal{N}(\delta, \mathcal{F}_{r}, L^{2}(P_{n}) }{n} }d\delta.
\end{equation}
Now it remains to control the covering number of the function space $\mathcal{N}(\delta, \mathcal{F}_{r}, L^{2}(P_{n}) )$. Actually, the covering number of the function space $\mathcal{N}(\delta, \mathcal{F}_{r}, L^{2}(P_{n}) )$ can be bounded by the covering number of the parameter space $\mathcal{N}(\delta_{\theta}, \varTheta_{r}, \norm{\cdot}_{F} )$ via the Lipschitz property with respect to parameters. Specifically, let $\mathcal{C}_{\varTheta_{r}}$ be a $\delta_{\theta}$-cover of $\varTheta_{r}$, such that for each $\Theta \in \varTheta_{r}$, there exists $\Theta' \in \mathcal{C}_{\varTheta_{r}}$ satisfying $\norm{\Theta' - \Theta}_{F} \leq \delta_{\theta}$. According to Lemma \ref{lipschitz-parameter}, one has that
\begin{equation*}
\| f_{\Theta}(x) - f_{\Theta'}(x) \|_{L^{2}(P_{n})} \leq \text{Lip}_{L^{2}(P_{n})}(x) \norm{\Theta-\Theta'}_{F}  \leq  \delta_{\theta}\text{Lip}_{L^{2}(P_{n})}(x).
\end{equation*}
Thus the function set $\mathcal{C}_{\mathcal{F}_{r}}=\{f_{\Theta'}:\Theta' \in \mathcal{C}_{\varTheta_{r}}  \}$ is a $\delta$-cover of $\mathcal{F}_{r}$ with $\delta=\delta_{\theta}\text{Lip}_{L^{2}(P_{n})}(x)$ and it holds that
\begin{equation*}
\mathcal{N}(\delta, \mathcal{F}_{r}, L^{2}(P_{n}) ) \leq  \mathcal{N}(\delta / \text{Lip}_{L^{2}(P_{n})}(x), \varTheta_{r}, \norm{\cdot}_{F} ).
\end{equation*}
On the other hand, for the parameter set $\Theta \in \varTheta_{r}$, let $\Theta_{\text{vec}}$ denote the vector formed by the vectorized form of each parameter matrix arranging one by one. Denote the total number of parameters as $P$. Since $\| \Theta_{\text{vec}} \|_{1} = \norm{\Theta}_{1} \leq r$, the transformed parameter set is
$ \mathbb{B}_{1,r}^{P} = \{\theta \in \mathbb{R}^{P}:  \| \theta \|_{1} \leq r\} $ and 
\begin{equation*}
\mathcal{N}(\delta_{\theta}, \varTheta_{r}, \norm{\cdot}_{F} ) = \mathcal{N}(\delta_{\theta}, \mathbb{B}_{1,r}^{P}, \| \cdot \|_{2} ).
\end{equation*}
Let $g \sim N(0,\mathbb{I}_{P})$ be a $P$-dimensional normal random vevtor. It follows from Sudakov's minoration inequality that
\begin{equation*}
\begin{aligned}
\delta_{\theta} \sqrt{\log \mathcal{N}(\delta_{\theta}, \mathbb{B}_{1,r}^{P}, \| \cdot \|_{2} )} &\leq 2\mathbb{E}\sup_{\theta \in \mathbb{B}_{1,r}^{P}} \langle \theta,g \rangle \leq 2\mathbb{E} \| \theta \|_{1} \|g\|_{\infty} \\
&\leq 2r\mathbb{E} \|g\|_{\infty} \leq 2r \sqrt{2\log P}.
\end{aligned}
\end{equation*}
Noting from \eqref{dudley}, we still need to bound the superior absolute value of $f_{\Theta} \in \mathcal{F}_{r}$ before calculating the Dudley's integral. According to \eqref{lip-model2} and Assumption \ref{bounded-inputs},
\begin{equation}\label{sup-f}
\sup_{x \in \mathcal{X}} | f_{\Theta}(x) | \leq \sup_{x \in \mathcal{X}} \| x \|_{\infty} \prod_{k=1}^{L} \| \theta_{k} \|_{1,1} \leq R\left(\frac{1}{L} \sum_{k=1}^{L} \| \theta_{k} \|_{1,1}\right)^{L} \leq R\left(\frac{r}{L}\right)^{L}.
\end{equation}
Then the Dudley's integral \eqref{dudley} is bounded as follows
\begin{equation}\label{eq-rsfr}
\begin{aligned}
\mathcal{R}_{S}(\mathcal{F}_{r}) &\leq 4\alpha + 12\int_{\alpha}^{\sup| f_{\Theta}(x) |} \sqrt{ \frac{\log \mathcal{N}(\delta, \mathcal{F}_{r}, L^{2}(P_{n}) }{n} }d\delta \\
&\leq 4\alpha + 12\int_{\alpha}^{\sup| f_{\Theta}(x) |} \sqrt{ \frac{\log \mathcal{N}(\delta / \text{Lip}_{L^{2}(P_{n})}(x), \mathbb{B}_{1,r}^{P}, \| \cdot \|_{2} ) }{n} }d\delta \\
&\leq  4\alpha + 24r\text{Lip}_{L^{2}(P_{n})}(x) \sqrt{\frac{2\log P}{n}}\int_{\alpha}^{\sup| f_{\Theta}(x) |} \frac{1}{\delta}d\delta \\
&\leq  4\alpha + 24r\text{Lip}_{L^{2}(P_{n})}(x) \sqrt{\frac{2\log P}{n}} \log(\frac{\sup| f_{\Theta}(x) |}{\alpha}).
\end{aligned}
\end{equation}
Substituting $\text{Lip}_{L^{2}(P_{n})}(x)$ in Lemma \ref{lipschitz-parameter} (cf. \eqref{eq-LipL2}) into \eqref{eq-rsfr}, one sees that the right-hand side of \eqref{eq-rsfr} is minimized with respect to $\alpha$ by setting
\begin{equation*}
\alpha = 6r \sqrt{L} \left( \frac{r}{L-1} \right)^{L-1} \sqrt{ \frac{1}{n} \sum_{i=1}^{n} \| x_{i} \|_{\infty}^{2} } \sqrt{\frac{2\log P}{n}},
\end{equation*}
For simplicity, we ignore the data dependent term by setting
\begin{equation*}
\alpha = 6r \sqrt{L} \left( \frac{r}{L-1} \right)^{L-1} \sqrt{\frac{2\log P}{n}}.
\end{equation*}
Then combining \eqref{sup-f} and \eqref{eq-rsfr}, we finally obtain that
\begin{equation*}
\mathcal{R}_{S}(\mathcal{F}_{r}) \leq 24r \left( \frac{r}{L-1} \right)^{L-1} \sqrt{\frac{2L\log P}{n}} \left(1+\log(c_{1}\sqrt{n}) \sqrt{ \frac{1}{n} \sum_{i=1}^{n} \| x_{i} \|_{\infty}^{2} }  \right),
\end{equation*}
where
\begin{equation*}
c_{1} = \frac{R}{6rL^{\frac{3}{2}}\sqrt{2\log P}}.
\end{equation*}
In addition, the expectation of $\mathcal{R}_{S}(\mathcal{F}_{r})$ is bounded as
\begin{equation*}
\begin{aligned}
\mathcal{R}_{n}(\mathcal{F}_{r}) &= \E[\mathcal{R}_{S}(\mathcal{F}_{r})] \\
&\leq 24r \left( \frac{r}{L-1} \right)^{L-1} \sqrt{\frac{2L\log P}{n}} \left(1+\log(c_{1}\sqrt{n}) \sqrt{\E \| x \|_{\infty}^{2} }  \right).
\end{aligned}
\end{equation*}
The proof is complete.
\end{proof}

\section{Proof of Section \ref{sec-derivative}}\label{sec-appB}

\begin{proof}[Proof of Lemma \ref{lipschitz-input}]
It follows from the chain rule that the partial derivatives with respect to inputs equal to
\begin{equation*}
\nabla_{x}f_{\Theta}(x) = \theta_{1}^{\top}  h'_{1} \odot \theta_{2}^{\top} \cdots \theta_{L-1}^{\top} h'_{L-1} \odot \theta_{L}^{\top}.
\end{equation*}
Then the quantity $\| \nabla_{x}f_{\Theta}(x) \|_{1}$ is bounded layer by layer as follows
\begin{equation*}
\begin{aligned}
\| \nabla_{x}f_{\Theta}(x) \|_{1} &= \| \theta_{1}^{\top}  h'_{1} \odot \theta_{2}^{\top} \cdots \theta_{L-1}^{\top} h'_{L-1} \odot \theta_{L}^{\top} \|_{1} \\
& \leq \| \theta_{1} \|_{1,1}   \| h'_{1} \odot \theta_{2}^{\top} \cdots \theta_{L-1}^{\top} h'_{L-1} \odot \theta_{L}^{\top} \|_{\infty} \\
& \leq \| \theta_{1} \|_{1,1}   \| h'_{1} \odot \theta_{2}^{\top} \cdots \theta_{L-1}^{\top} h'_{L-1} \odot \theta_{L}^{\top} \|_{1} \\
& \leq \| \theta_{1} \|_{1,1}   \| \theta_{2}^{\top}  h'_{2} \cdots \theta_{L-1}^{\top} h'_{L-1} \odot \theta_{L}^{\top} \|_{1} \\
& \leq \prod_{k=1}^{L} \| \theta_{k} \|_{1,1} \leq \left(\frac{1}{L} \sum_{k=1}^{L} \| \theta_{k} \|_{1,1}\right)^{L} \\
& \leq \left(\frac{r}{L}\right)^{L},
\end{aligned}
\end{equation*}
where the first inequality follows from the fact that $\|Ax\|_{1} = \sum_{i} | \langle A_{i\cdot},x \rangle | \leq \sum_{i} \|A_{i\cdot}\|_{1} \|x \|_{\infty}$\\$=\|A\|_{1,1}\|x \|_{\infty} $. 

\end{proof}

\begin{proof}[Proof of Lemma \ref{bounded-divergence}]
For a deep neural network $f_\Theta(x)=\theta_L\sigma(\theta_{L-1}\cdots\sigma(\theta_1 x)\cdots)$, denote the activations of the $l$-th hidden layer as $h_{l}=\sigma(\theta_{l}\cdots\sigma(\theta_1 x)\cdots)$ with $h_{0}(x)=x$. Then the derivatives of the $l$-th hidden layer is $h'_{l}=\sigma'(\theta_{l}\cdots\sigma(\theta_1 x)\cdots)$, where $\sigma'(\cdot)$ is the derivative of the activation function. It follows that
\begin{equation*}
\begin{aligned}
\nabla_{x}f_{\Theta}(x) = \theta_{1}^{\top}  h'_{1} \odot \theta_{2}^{\top} \cdots \theta_{L-1}^{\top} h'_{L-1} \odot \theta_{L}^{\top}.
\end{aligned}
\end{equation*}
Recall that ${\theta_{1}}_{\cdot i}$ denote the $i$-th column of $\theta_{1}$. Then it follows that
\begin{equation*}
\begin{aligned}
\Delta_{x}f_{\Theta}(x) =\sum_{i=1}^{d}\nabla_{x_{i}}\nabla_{x_{i}}f_{\Theta}(x), \quad \nabla_{x_{i}}f_{\Theta}(x) = \theta_{1\cdot i}^{\top}  h'_{1} \odot \theta_{2}^{\top} \cdots \theta_{L-1}^{\top} h'_{L-1} \odot \theta_{L}^{\top}.
\end{aligned}
\end{equation*}
According to the chain rule, we obtain that
\begin{equation*}
\begin{aligned}
\nabla_{x_{i}}\nabla_{x_{i}}f_{\Theta}(x) = \sum_{k=1}^{L-1} \left\langle \frac{\partial \nabla_{x_{i}}f_{\Theta}(x)}{\partial h'_{k}}, \frac{\partial h'_{k}}{\partial x_{i}} \right\rangle,
\end{aligned}
\end{equation*}
and
\begin{equation*}
\begin{aligned}
\frac{\partial \nabla_{x_{i}}f_{\Theta}(x)}{\partial h'_{k}} = \left(\theta_{1\cdot i} \odot h'_{1}  \theta_{2} \cdots h'_{k-1}  \theta_{k} \right) \odot \left( \theta_{k+1}^{\top}  h'_{k+1} \cdots \theta_{L-1}^{\top} h'_{L-1} \odot \theta_{L}^{\top} \right).
\end{aligned}
\end{equation*}
Denote $h''_{k}=\sigma''(\theta_{k}\cdots\sigma(\theta_1 x)\cdots)$ and assume the $k$-th layer has $h$ hidden units, and we use $\theta_{kj\cdot}$ to denote the $j$-th row of $\theta_{k}$,
\begin{equation*}
\begin{aligned}
\frac{\partial h'_{k}}{\partial x_{i}} = ( (\nabla_{x_{i}}h'_{k})_{1}, \cdots ,(\nabla_{x_{i}}h'_{k})_{h} )^{\top},
\end{aligned}
\end{equation*}

\begin{equation*}
\begin{aligned}
(\nabla_{x_{i}}h'_{k})_{j} = h''_{kj}  \theta_{kj\cdot}^{\top} \odot h'_{k-1}  \theta_{k-1}^{\top} \cdots \odot h'_{1}  \theta_{1\cdot i}^{\top}.
\end{aligned}
\end{equation*}
Then the divergence of partial derivatives is bounded as follows
\begin{equation}\label{divergence-ineq1}
| \Delta_{x}f_{\Theta}(x) | = \left|\sum_{i=1}^{d} \sum_{k=1}^{L-1} \left\langle \frac{\partial \nabla_{x_{i}}f_{\Theta}(x)}{\partial h'_{k}}, \frac{\partial h'_{k}}{\partial x_{i}} \right\rangle \right| \leq \sum_{i=1}^{d} \sum_{k=1}^{L-1}  \left\| \frac{\partial \nabla_{x_{i}}f_{\Theta}(x)}{\partial h'_{k}}\right\|_{2} \left\| \frac{\partial h'_{k}}{\partial x_{i}}\right\|_{2}.
\end{equation}
By the fact that $\|a \odot b\|_{2} =\sqrt{ \sum_{i}a_{i}^{2}b_{i}^{2} } \leq \sqrt{(\sum_{i}a_{i}^{2})(\sum_{i}b_{i}^{2})} = \|a\|_{2} \|b\|_{2}$ for two vectors $a,b$ with the same dimension, we have that
\begin{equation*}
\begin{aligned}
\left\| \frac{\partial \nabla_{x_{i}}f_{\Theta}(x)}{\partial h'_{k}}\right\|_{2} \leq  \| \theta_{1\cdot i} \odot h'_{1}  \theta_{2} \cdots h'_{k-1}  \theta_{k} \|_{2}  \| \theta_{k+1}^{\top}  h'_{k+1} \cdots \theta_{L-1}^{\top} h'_{L-1} \odot \theta_{L}^{\top} \|_{2}.
\end{aligned}
\end{equation*}
Meanwhile, the right-hand side of the former inequality is bounded by
\begin{equation*}
\begin{aligned}
\| \theta_{1\cdot i} \odot h'_{1}  \theta_{2} \cdots h'_{k-1}  \theta_{k} \|_{2} &\leq \| \theta_{1\cdot i} \odot h'_{1}  \theta_{2} \cdots \theta_{k-1}\odot h'_{k-1} \|_{\infty}   \| \theta_{k} \|_{1,1} \\
&\leq \| \theta_{1\cdot i} \odot h'_{1}  \theta_{2} \cdots \theta_{k-1} \|_{\infty}   \| \theta_{k} \|_{1,1} \\
&\leq \| \theta_{1\cdot i} \odot h'_{1}  \theta_{2} \cdots \theta_{k-1} \|_{2}   \| \theta_{k} \|_{1,1} \\
&\leq \| \theta_{1\cdot i}\|_{\infty} \prod_{q=2}^{k}\| \theta_{q} \|_{1,1} \leq \| \theta_{1\cdot i}\|_{1} \prod_{q=2}^{k}\| \theta_{q} \|_{1,1},
\end{aligned}
\end{equation*}
and
\begin{equation*}
\begin{aligned}
\| \theta_{k+1}^{\top}  h'_{k+1} \cdots \theta_{L-1}^{\top} h'_{L-1} \odot \theta_{L}^{\top} \|_{2} &\leq \| \theta_{k+1} \|_{1,1} \| h'_{k+1} \odot \theta_{k+2}^{\top} \cdots \theta_{L-1}^{\top} h'_{L-1} \odot \theta_{L}^{\top} \|_{\infty} \\
&\leq \| \theta_{k+1} \|_{1,1} \| \theta_{k+2}^{\top} \cdots \theta_{L-1}^{\top} h'_{L-1} \odot \theta_{L}^{\top} \|_{2} \\
&\leq \prod_{q=k+1}^{L}\| \theta_{q} \|_{1,1}.
\end{aligned}
\end{equation*}
Then it follows that
\begin{equation}\label{divergence-ineq2}
\left\| \frac{\partial \nabla_{x_{i}}f_{\Theta}(x)}{\partial h'_{k}}\right\|_{2} \leq \| \theta_{1\cdot i}\|_{1} \prod_{q=2}^{k}\| \theta_{q} \|_{1,1} \prod_{q=k+1}^{L}\| \theta_{q} \|_{1,1} = \| \theta_{1\cdot i}\|_{1} \prod_{q=2}^{L}\| \theta_{q} \|_{1,1}.
\end{equation}
Since we adopt softplus activation function $\sigma(\cdot)$, $\sigma'(\cdot)$ is sigmoid function and $\sigma''(\cdot) = \sigma'(\cdot)(1-\sigma'(\cdot))$, thus $\sigma''(\cdot) \in (0,1/4)$. Denote $\text{diag}(h'_{k})$ to be a matrix with diagonal elements equaling to $h'_{k}$. Then the quantity $(\nabla_{x_{i}}h'_{k})_{j}$ is bounded as follows
\begin{equation*}
\begin{aligned}
|(\nabla_{x_{i}}h'_{k})_{j}| &= | h''_{kj}  \theta_{kj\cdot}^{\top} \odot h'_{k-1}  \theta_{k-1}^{\top} \cdots \odot h'_{1}  \theta_{1\cdot i}^{\top} | \\
&= | h''_{kj}  \theta_{kj\cdot}^{\top}  \text{diag}(h'_{k-1})  \theta_{k-1}^{\top} \cdots \text{diag}(h'_{1})  \theta_{1\cdot i}^{\top} | \\
&\leq \frac{1}{4} \| \theta_{kj\cdot}^{\top}  \text{diag}(h'_{k-1}) \|_{1}  \| \theta_{k-1}^{\top}  \text{diag}(h'_{k-2}) \cdots \text{diag}(h'_{1})  \theta_{1\cdot i}^{\top} \|_{\infty} \\
&\leq \frac{1}{4} \| \theta_{kj\cdot} \|_{1}  \| \theta_{k-1}^{\top}  \text{diag}(h'_{k-2}) \cdots \text{diag}(h'_{1})  \theta_{1\cdot i}^{\top} \|_{2} \\
&\leq \frac{1}{4} \| \theta_{kj\cdot} \|_{1}  \| \theta_{k-1}^{\top}  \text{diag}(h'_{k-2})\|_{1,1} \|\theta_{k-2}^{\top}  \text{diag}(h'_{k-3}) \cdots  \theta_{1\cdot i}^{\top} \|_{\infty} \\
&\leq \frac{1}{4} \| \theta_{kj\cdot} \|_{1}  \| \theta_{k-1}\|_{1,1} \|\theta_{k-2}^{\top}  \text{diag}(h'_{k-3}) \cdots  \theta_{1\cdot i}^{\top} \|_{2} \\
&\leq \frac{1}{4} \| \theta_{kj\cdot} \|_{1}  \prod_{q=2}^{k-1} \| \theta_{q}\|_{1,1} \|\theta_{1\cdot i}^{\top} \|_{\infty} \leq \frac{1}{4} \| \theta_{kj\cdot} \|_{1} \|\theta_{1\cdot i} \|_{1}  \prod_{q=2}^{k-1} \| \theta_{q}\|_{1,1} .
\end{aligned}
\end{equation*}
Then it follows that
\begin{equation}\label{divergence-ineq3}
\begin{aligned}
\left\| \frac{\partial h'_{k}}{\partial x_{i}}\right\|_{2} &= \sqrt{\sum_{j=1}^{h}(\nabla_{x_{i}}h'_{k})_{j}^{2}} \leq \sqrt{ \sum_{j=1}^{h} \frac{1}{16} \| \theta_{kj\cdot} \|_{1}^{2} \|\theta_{1\cdot i} \|_{1}^{2}  \prod_{q=2}^{k-1} \| \theta_{q}\|_{1,1}^{2} } \\
&\leq \sqrt{ \left(\sum_{j=1}^{h} \| \theta_{kj\cdot} \|_{1}\right)^{2} \frac{1}{16} \|\theta_{1\cdot i} \|_{1}^{2}  \prod_{q=2}^{k-1} \| \theta_{q}\|_{1,1}^{2} } \\
&=\frac{1}{4} \|\theta_{1\cdot i} \|_{1}  \prod_{q=2}^{k} \| \theta_{q}\|_{1,1}.
\end{aligned}
\end{equation}
Combing \eqref{divergence-ineq1},\eqref{divergence-ineq2} and \eqref{divergence-ineq3}, we finally arrive at that
\begin{equation*}
\begin{aligned}
| \Delta_{x}f_{\Theta}(x) | &\leq \sum_{i=1}^{d} \sum_{k=1}^{L-1}  (\| \theta_{1\cdot i}\|_{1} \prod_{q=2}^{L}\| \theta_{q} \|_{1,1}) ( \frac{1}{4} \|\theta_{1\cdot i} \|_{1}  \prod_{q=2}^{k} \| \theta_{q}\|_{1,1}) \\
&=\frac{1}{4}\prod_{q=2}^{L}\| \theta_{q} \|_{1,1} \left(\sum_{i=1}^{d} \| \theta_{1\cdot i}\|_{1}^{2}\right) \left(\sum_{k=1}^{L-1}  \prod_{q=2}^{k} \| \theta_{q}\|_{1,1}\right) \\
&\leq \frac{1}{4}\prod_{q=2}^{L}\| \theta_{q} \|_{1,1} \left(\sum_{i=1}^{d} \| \theta_{1\cdot i}\|_{1}\right)^{2} \left(\sum_{k=1}^{L-1}  \prod_{q=2}^{k} \| \theta_{q}\|_{1,1}\right) \\
&= \frac{1}{4}\prod_{q=2}^{L}\| \theta_{q} \|_{1,1} (\|\theta_{1}\|_{1,1}^{2}) \left(\sum_{k=2}^{L-1}  \prod_{q=2}^{k} \| \theta_{q}\|_{1,1}\right) \\
&= \frac{1}{4}\prod_{q=1}^{L}\| \theta_{q} \|_{1,1} \left(\sum_{k=2}^{L-1}  \prod_{q=1}^{k} \| \theta_{q}\|_{1,1}\right) \\
&\leq \frac{1}{4}\left(\frac{r}{L}\right)^{L} \left((L-2)\max_{k\in\{2,\cdots,L-1\}} \prod_{q=1}^{k} \| \theta_{q}\|_{1,1} \right)\\
&\leq \frac{L}{4}\left(\frac{r}{L}\right)^{L} \max_{k\in\{2,\cdots,L-1\}} \left( \frac{r}{k} \right)^{k}.
\end{aligned}
\end{equation*}
The proof is complete.
\end{proof}

\begin{proof}[Proof of Lemma \ref{green-lemma}]
The proof mainly follows \cite[Lemma A1]{RN357}, and we provide it here for the completeness of this article. By the definition of the inner product, one has that
\begin{equation*}
\begin{aligned}
&-\langle \nabla_{x}f,\nabla_{x}g \rangle_{L^{2}} = - \int_{\mathcal{X}}(\nabla_{x}f \cdot \nabla_{x}g)p_x(x)dx \\
&=- \int_{\mathcal{X}}\nabla_{x} \cdot ((\nabla_{x}f )gp_x(x))dx + \int_{\mathcal{X}}\nabla_{x} \cdot ((\nabla_{x}f )p_x(x))gdx \quad (\text{via integral by part})\\
&=- \int_{\partial \mathcal{X}} (\nabla_{x}f \cdot \vec{n}) gp_x(x)ds + \int_{\mathcal{X}}\nabla_{x} \cdot ((\nabla_{x}f )p_x(x))gdx \quad (\text{via divergence theorem})\\
&=0+\int_{\mathcal{X}} (\nabla_{x} \cdot \nabla_{x}f) g p_x(x)dx+\int_{\mathcal{X}} (\nabla_{x}f \cdot \nabla_{x}p_x(x)) g dx \quad (\because \nabla_{x}f\cdot\vec{n}=0)\\
&=\int_{\mathcal{X}} (\Delta_{x}f) g p_x(x)dx+\int_{\mathcal{X}} (\nabla_{x}f \cdot \nabla_{x}\log p_x(x)) g p_x(x)dx \\
&= \langle \Delta_{x}f+\nabla_{x}f \cdot \nabla_{x} \log p_x(x), g \rangle_{L^{2}}.
\end{aligned}
\end{equation*}
\end{proof}

\begin{proof}[Proof of Theorem \ref{convergence-gradient}]
By the linearity of the inner product, we have that
\begin{equation}\label{eq-thm2-1}
\begin{aligned}
\|\nabla_{x}\hat{f} - \nabla_{x}f_{0}\|_{L^{2}}^{2} &= \langle \nabla_{x}\hat{f} - \nabla_{x}f_{0}, \nabla_{x}\hat{f} - \nabla_{x}f_{0} \rangle _{L^{2}} \\
&=\langle \nabla_{x}\hat{f}, \nabla_{x}\hat{f} - \nabla_{x}f_{0} \rangle _{L^{2}} -\langle \nabla_{x}f_{0}, \nabla_{x}\hat{f} - \nabla_{x}f_{0} \rangle _{L^{2}}.
\end{aligned}
\end{equation}
It follows from Assumption \ref{boundary-assumption} and Lemma \ref{green-lemma} that 
\begin{equation*}
\langle \nabla_{x}\hat{f}, \nabla_{x}\hat{f} - \nabla_{x}f_{0} \rangle _{L^{2}} = - \langle \Delta_{x}\hat{f}+\nabla_{x}\hat{f} \cdot \nabla_{x} \log p_x(x), \hat{f}-f_{0} \rangle_{L^{2}},
\end{equation*}
\begin{equation*}
-\langle \nabla_{x}f_{0}, \nabla_{x}\hat{f} - \nabla_{x}f_{0} \rangle _{L^{2}} =  \langle \Delta_{x}f_{0}+\nabla_{x}f_{0} \cdot \nabla_{x} \log p_x(x), \hat{f}-f_{0} \rangle_{L^{2}}.
\end{equation*}
Combining the former two equalities with \eqref{eq-thm2-1} yields that
\begin{equation}\label{convergence-derivative1}
\begin{aligned}
\|\nabla_{x}\hat{f} - \nabla_{x}f_{0}\|_{L^{2}}^{2} &= \langle \Delta_{x}f_{0} - \Delta_{x}\hat{f}, \hat{f}-f_{0} \rangle_{L^{2}} \\
&+ \langle (\nabla_{x}f_{0}-\nabla_{x}\hat{f}) \cdot \nabla_{x} \log p_x(x), \hat{f}-f_{0} \rangle_{L^{2}}.
\end{aligned}
\end{equation}
Set $\lambda = 1/ \sqrt{n}$. Then the first term in the right-hand side of \eqref{convergence-derivative1} is bounded as 
\begin{equation*}
\begin{aligned}
\langle \Delta_{x}f_{0} - \Delta_{x}\hat{f}, \hat{f}-f_{0} \rangle_{L^{2}} &\leq \| \Delta_{x}f_{0} - \Delta_{x}\hat{f} \|_{L^{2}}  \| \hat{f}-f_{0} \|_{L^{2}} \\
&\leq \frac{\sqrt{\lambda}}{2}\| \Delta_{x}f_{0} - \Delta_{x}\hat{f} \|_{L^{2}}^{2} + \frac{1}{2\sqrt{\lambda}}\| \hat{f}-f_{0} \|_{L^{2}}^{2} \\
&\leq \frac{\sqrt{\lambda}}{2} (\| \Delta_{x}f_{0} \|_{L^{2}} + \| \Delta_{x}\hat{f} \|_{L^{2}})^{2} + \frac{1}{2\sqrt{\lambda}}\| \hat{f}-f_{0} \|_{L^{2}}^{2} \\
&\leq \sqrt{\lambda} (\| \Delta_{x}f_{0} \|_{L^{2}}^{2} + \| \Delta_{x}\hat{f} \|_{L^{2}}^{2}) + \frac{1}{2\sqrt{\lambda}}\| \hat{f}-f_{0} \|_{L^{2}}^{2}.
\end{aligned}
\end{equation*}
On the other hand, using Lemma \ref{bounded-divergence} to bound the divergence of partial derivatives, we have that for $\Theta \in \varTheta_{r}$,
\begin{equation*}
\begin{aligned}
\| \Delta_{x}f_{\Theta} \|_{L^{2}}^{2} = \int_{\mathcal{X}}
| \Delta_{x}f_{\Theta}(x) |^{2} p_x(x) dx &\leq \int_{\mathcal{X}} \frac{L^{2}}{16}\left(\frac{r}{L}\right)^{2L} \max_{k\in\{2,\cdots,L-1\}} \left( \frac{r}{k} \right)^{2k} p_x(x) dx \\
&=\frac{L^{2}}{16}\left(\frac{r}{L}\right)^{2L} \max_{k\in\{2,\cdots,L-1\}} \left( \frac{r}{k} \right)^{2k}.
\end{aligned}
\end{equation*}
Combining the former two inequalities, we obtain that
\begin{equation}\label{convergence-derivativeF}
\begin{aligned}
\langle \Delta_{x}f_{0} - \Delta_{x}\hat{f}, \hat{f}-f_{0} \rangle_{L^{2}} \leq \frac{\sqrt{\lambda}L^{2}}{8}\left(\frac{r}{L}\right)^{2L} \max_{k\in\{2,\cdots,L-1\}} \left( \frac{r}{k} \right)^{2k} + \frac{1}{2\sqrt{\lambda}}\| \hat{f}-f_{0} \|_{L^{2}}^{2}.
\end{aligned}
\end{equation}
For the second term in the right-hand side of \eqref{convergence-derivative1}, one has that
\begin{equation}\label{convergence-derivative2}
\begin{aligned}
&\langle (\nabla_{x}f_{0}-\nabla_{x}\hat{f}) \cdot \nabla_{x} \log p_x(x), \hat{f}-f_{0} \rangle_{L^{2}} \\
&= \int_{\mathcal{X}} \left( (\nabla_{x}f_{0}(x)-\nabla_{x}\hat{f}(x)) \cdot \nabla_{x} \log p_x(x) \right) (\hat{f}(x)-f_{0}(x)) p_x(x)dx \\
&\leq \int_{\mathcal{X}} \| \nabla_{x}f_{0}(x)-\nabla_{x}\hat{f}(x) \|_{1}  \| \nabla_{x} \log p_x(x)\|_{\infty} (\hat{f}(x)-f_{0}(x)) p_x(x)dx \\
&=\langle \| \nabla_{x}f_{0}(x)-\nabla_{x}\hat{f}(x) \|_{1}, \| \nabla_{x} \log p_x(x)\|_{\infty} (\hat{f}(x)-f_{0}(x)) \rangle_{L^{2}} \\
&\leq \left\| \| \nabla_{x}f_{0}(x)-\nabla_{x}\hat{f}(x) \|_{1} \right\|_{L^{2}}  \left\| \| \nabla_{x} \log p_x(x)\|_{\infty} (\hat{f}(x)-f_{0}(x)) \right\|_{L^{2}} \\
&\leq \frac{\sqrt{\lambda}}{2}\left\| \| \nabla_{x}f_{0}(x)-\nabla_{x}\hat{f}(x) \|_{1} \right\|_{L^{2}}^{2} + \frac{1}{2\sqrt{\lambda}} \left\| \| \nabla_{x} \log p_x(x)\|_{\infty} (\hat{f}(x)-f_{0}(x)) \right\|_{L^{2}}^{2}.
\end{aligned}
\end{equation}
Using Lemma \ref{lipschitz-input} to bound the $\ell_{1}$-norm of partial derivatives, one sees that the first term in the right-hand side of \eqref{convergence-derivative2} can be bounded as
\begin{equation}\label{convergence-derivative3}
\begin{aligned}
\left\| \| \nabla_{x}f_{0}(x)-\nabla_{x}\hat{f}(x) \|_{1} \right\|_{L^{2}}^{2} &\leq \left\| \|\nabla_{x}f_{0}(x)\|_{1} + \|\nabla_{x}\hat{f}(x) \|_{1} \right\|_{L^{2}}^{2} \\
&= \int_{\mathcal{X}} (\|\nabla_{x}f_{0}(x)\|_{1} + \|\nabla_{x}\hat{f}(x) \|_{1})^{2} p_x(x)dx \\
&\leq \int_{\mathcal{X}} (2(r/L)^{L})^{2} p_x(x)dx \\
&=4\left(\frac{r}{L}\right)^{2L}.
\end{aligned}
\end{equation}
For the second term in the right-hand side of \eqref{convergence-derivative2}, recalling the $L^{\infty}$-norm given by $\|g\|_{L^{\infty}} = \sup\{g(x): x \in \mathcal{X}\}$, we obtain that
\begin{equation}\label{convergence-derivative4}
\begin{aligned}
&\left\| \| \nabla_{x} \log p_x(x)\|_{\infty} (\hat{f}(x)-f_{0}(x)) \right\|_{L^{2}}^{2} \\
&=\int_{\mathcal{X}} \| \nabla_{x} \log p_x(x)\|_{\infty}^{2} (\hat{f}(x)-f_{0}(x))^{2} p_x(x)dx \\
& \leq \left\| \| \nabla_{x} \log p_x(x)\|_{\infty}^{2} \right\|_{L^{\infty}} \int_{\mathcal{X}} (\hat{f}(x)-f_{0}(x))^{2} p_x(x)dx \\
& = \left\| \| \nabla_{x} \log p_x(x)\|_{\infty}^{2} \right\|_{L^{\infty}} \|\hat{f}-f_{0}\|_{L^{2}}^{2} \\
& \leq b_{1}^{2}\|\hat{f}-f_{0}\|_{L^{2}}^{2}.
\end{aligned}
\end{equation}
where the last inequality is due to Assumption \ref{bounded-density} that bounds the superior value of derivatives of the probability density. Combining \eqref{convergence-derivative2},\eqref{convergence-derivative3} and \eqref{convergence-derivative4}, it follows that
\begin{equation}\label{convergence-derivativeS}
\langle (\nabla_{x}f_{0}-\nabla_{x}\hat{f}) \cdot \nabla_{x} \log p_x(x), \hat{f}-f_{0} \rangle_{L^{2}} \leq 2\sqrt{\lambda}\left(\frac{r}{L}\right)^{2L} + \frac{b_{1}^{2}}{2\sqrt{\lambda}}\|\hat{f}-f_{0}\|_{L^{2}}^{2}.
\end{equation}
Combining \eqref{convergence-derivative1},\eqref{convergence-derivativeF} and \eqref{convergence-derivativeS}, we have that
\begin{equation*}
\begin{aligned}
\|\nabla_{x}\hat{f} - \nabla_{x}f_{0}\|_{L^{2}}^{2} &\leq \sqrt{\lambda}\left(\frac{r}{L}\right)^{2L}\left(2+ \frac{L^{2}}{8} \max_{k\in\{2,\cdots,L-1\}} \left( \frac{r}{k} \right)^{2k}\right) \\
&+ \frac{1+b_{1}^{2}}{2\sqrt{\lambda}}\| \hat{f}-f_{0} \|_{L^{2}}^{2}.
\end{aligned}
\end{equation*}
Taking expectations with respect to the dataset $\D$ and plugging $\lambda = 1/ \sqrt{n}$ into the former inequality, we finally obtain by Theorem \ref{model-convergence} that
\begin{equation*}
\begin{aligned}
&\E_{\D} \|\nabla_{x}\hat{f} - \nabla_{x}f_{0}\|_{L^{2}}^{2} \leq \frac{1}{n^{1/4}}\left(\frac{r}{L}\right)^{2L}\left(2+ \frac{L^{2}}{8} \max_{k\in\{2,\cdots,L-1\}} \left( \frac{r}{k} \right)^{2k}\right) \\
&+  48(1+b_{1}^{2})b_{0}r \left( \frac{r}{L-1} \right)^{L-1} \frac{\sqrt{2L\log P}}{n^{1/4}} \left(1+\log(c_{1}\sqrt{n}) \sqrt{\E \| x \|_{\infty}^{2} }  \right).
\end{aligned}
\end{equation*}
The proof is complete.
\end{proof}
\end{appendix}

%
%
%
%

\begin{acks}[Acknowledgments]
The authors would like to thank the anonymous referees, an Associate Editor and the Editor for their constructive comments that improved the
quality of this paper.
\end{acks}

\begin{funding}
The first author was supported by the Natural Science Foundation of China (Grant No. 62103329). 
The second author was supported by the Natural Science Foundation of China (Grant No. 12201496).
\end{funding}

\bibliographystyle{imsart-nameyear.bst}
\bibliography{./arxiv-240609.bbl}

\begin{thebibliography}{33}

\bibitem[\protect\citeauthoryear{Abramovich}{2023}]{RN333}
\begin{barticle}[author]
\bauthor{\bsnm{Abramovich},~\bfnm{Felix}\binits{F.}}
(\byear{2023}).
\btitle{Statistical learning by sparse deep neural networks}.
\bjournal{ArXiv}.
\end{barticle}
\endbibitem

\bibitem[\protect\citeauthoryear{Allen}{2013}]{RN310}
\begin{barticle}[author]
\bauthor{\bsnm{Allen},~\bfnm{Genevera~I.}\binits{G.~I.}}
(\byear{2013}).
\btitle{Automatic Feature Selection via Weighted Kernels and Regularization}.
\bjournal{Journal of Computational and Graphical Statistics}
\bvolume{22}
\bpages{284-299}.
\bdoi{10.1080/10618600.2012.681213}
\end{barticle}
\endbibitem

\bibitem[\protect\citeauthoryear{Bach}{2017}]{RN70}
\begin{barticle}[author]
\bauthor{\bsnm{Bach},~\bfnm{Francis~R}\binits{F.~R.}}
(\byear{2017}).
\btitle{Breaking the Curse of Dimensionality with Convex Neural Networks}.
\bjournal{Journal of Machine Learning Research}
\bvolume{18}
\bpages{1-53}.
\end{barticle}
\endbibitem

\bibitem[\protect\citeauthoryear{Bartlett}{1998}]{RN380}
\begin{barticle}[author]
\bauthor{\bsnm{Bartlett},~\bfnm{P.~L.}\binits{P.~L.}}
(\byear{1998}).
\btitle{The sample complexity of pattern classification with neural networks:
  the size of the weights is more important than the size of the network}.
\bjournal{IEEE Transactions on Information Theory}
\bvolume{44}
\bpages{525-536}.
\bdoi{10.1109/18.661502}
\end{barticle}
\endbibitem

\bibitem[\protect\citeauthoryear{Bartlett, Foster and Telgarsky}{2017}]{RN481}
\begin{binproceedings}[author]
\bauthor{\bsnm{Bartlett},~\bfnm{Peter~L.}\binits{P.~L.}},
  \bauthor{\bsnm{Foster},~\bfnm{Dylan~J.}\binits{D.~J.}} \AND
  \bauthor{\bsnm{Telgarsky},~\bfnm{Matus}\binits{M.}}
(\byear{2017}).
\btitle{Spectrally-normalized margin bounds for neural networks}.
In \bbooktitle{Neural Information Processing Systems}.
\end{binproceedings}
\endbibitem

\bibitem[\protect\citeauthoryear{Bauer and Kohler}{2019}]{RN343}
\begin{barticle}[author]
\bauthor{\bsnm{Bauer},~\bfnm{Benedikt}\binits{B.}} \AND
  \bauthor{\bsnm{Kohler},~\bfnm{Michael}\binits{M.}}
(\byear{2019}).
\btitle{On deep learning as a remedy for the curse of dimensionality in
  nonparametric regression}.
\bjournal{The Annals of Statistics}
\bvolume{47}
\bpages{2261-2285}.
\end{barticle}
\endbibitem

\bibitem[\protect\citeauthoryear{Bertin and Lecué}{2008}]{RN307}
\begin{barticle}[author]
\bauthor{\bsnm{Bertin},~\bfnm{Karine}\binits{K.}} \AND
  \bauthor{\bsnm{Lecué},~\bfnm{Guillaume}\binits{G.}}
(\byear{2008}).
\btitle{Selection of variables and dimension reduction in high-dimensional
  non-parametric regression}.
\bjournal{Electronic Journal of Statistics}
\bvolume{2}
\bpages{1224-1241}.
\end{barticle}
\endbibitem

\bibitem[\protect\citeauthoryear{Chen et~al.}{2017}]{RN311}
\begin{barticle}[author]
\bauthor{\bsnm{Chen},~\bfnm{Jingxiang}\binits{J.}},
  \bauthor{\bsnm{Zhang},~\bfnm{Chong}\binits{C.}},
  \bauthor{\bsnm{Kosorok},~\bfnm{Michael~R.}\binits{M.~R.}} \AND
  \bauthor{\bsnm{Liu},~\bfnm{Yufeng}\binits{Y.}}
(\byear{2017}).
\btitle{Double Sparsity Kernel Learning with Automatic Variable Selection and
  Data Extraction}.
\bjournal{Statistics and Its Interface}
\bvolume{11}
\bpages{401-420}.
\end{barticle}
\endbibitem

\bibitem[\protect\citeauthoryear{Ding et~al.}{2024}]{RN357}
\begin{barticle}[author]
\bauthor{\bsnm{Ding},~\bfnm{Zhao}\binits{Z.}},
  \bauthor{\bsnm{Duan},~\bfnm{Chenguang}\binits{C.}},
  \bauthor{\bsnm{Jiao},~\bfnm{Yuling}\binits{Y.}} \AND
  \bauthor{\bsnm{Yang},~\bfnm{Jerry~Zhijian}\binits{J.~Z.}}
(\byear{2024}).
\btitle{Semi-Supervised Deep Sobolev Regression: Estimation, Variable Selection
  and Beyond}.
\bjournal{ArXiv}.
\end{barticle}
\endbibitem

\bibitem[\protect\citeauthoryear{Dinh and Ho}{2020}]{RN338}
\begin{binproceedings}[author]
\bauthor{\bsnm{Dinh},~\bfnm{Vu~C.}\binits{V.~C.}} \AND
  \bauthor{\bsnm{Ho},~\bfnm{Lam Si~Tung}\binits{L.~S.~T.}}
(\byear{2020}).
\btitle{Consistent Feature Selection for Analytic Deep Neural Networks}.
In \bbooktitle{Neural Information Processing Systems}.
\end{binproceedings}
\endbibitem

\bibitem[\protect\citeauthoryear{Fan and Li}{2001}]{RN320}
\begin{barticle}[author]
\bauthor{\bsnm{Fan},~\bfnm{Jianqing}\binits{J.}} \AND
  \bauthor{\bsnm{Li},~\bfnm{Runze}\binits{R.}}
(\byear{2001}).
\btitle{Variable Selection via Nonconcave Penalized Likelihood and its Oracle
  Properties}.
\bjournal{Journal of the American Statistical Association}
\bvolume{96}
\bpages{1348-1360}.
\bdoi{10.1198/016214501753382273}
\end{barticle}
\endbibitem

\bibitem[\protect\citeauthoryear{Feng and Simon}{2017}]{RN274}
\begin{barticle}[author]
\bauthor{\bsnm{Feng},~\bfnm{Jean}\binits{J.}} \AND
  \bauthor{\bsnm{Simon},~\bfnm{Noah~R}\binits{N.~R.}}
(\byear{2017}).
\btitle{Sparse-Input Neural Networks for High-dimensional Nonparametric
  Regression and Classification}.
\bjournal{ArXiv}.
\end{barticle}
\endbibitem

\bibitem[\protect\citeauthoryear{Fukumizu}{1996}]{RN495}
\begin{barticle}[author]
\bauthor{\bsnm{Fukumizu},~\bfnm{Kenji}\binits{K.}}
(\byear{1996}).
\btitle{A Regularity Condition of the Information Matrix of a Multilayer
  Perceptron Network}.
\bjournal{Neural Networks}
\bvolume{9}
\bpages{871-879}.
\bdoi{https://doi.org/10.1016/0893-6080(95)00119-0}
\end{barticle}
\endbibitem

\bibitem[\protect\citeauthoryear{Gayraud and Ingster}{2012}]{RN309}
\begin{barticle}[author]
\bauthor{\bsnm{Gayraud},~\bfnm{Ghislaine}\binits{G.}} \AND
  \bauthor{\bsnm{Ingster},~\bfnm{Yuri}\binits{Y.}}
(\byear{2012}).
\btitle{Detection of sparse additive functions}.
\bjournal{Electronic Journal of Statistics}
\bvolume{6}
\bpages{1409-1448}.
\end{barticle}
\endbibitem

\bibitem[\protect\citeauthoryear{Golowich, Rakhlin and Shamir}{2018}]{RN488}
\begin{binproceedings}[author]
\bauthor{\bsnm{Golowich},~\bfnm{Noah}\binits{N.}},
  \bauthor{\bsnm{Rakhlin},~\bfnm{Alexander}\binits{A.}} \AND
  \bauthor{\bsnm{Shamir},~\bfnm{Ohad}\binits{O.}}
(\byear{2018}).
\btitle{Size-Independent Sample Complexity of Neural Networks}.
In \bbooktitle{Annual Conference Computational Learning Theory}
\bvolume{75}
\bpages{297-299}.
\end{binproceedings}
\endbibitem

\bibitem[\protect\citeauthoryear{Ho and Dinh}{2020}]{RN341}
\begin{barticle}[author]
\bauthor{\bsnm{Ho},~\bfnm{Lam Si~Tung}\binits{L.~S.~T.}} \AND
  \bauthor{\bsnm{Dinh},~\bfnm{Vu~C.}\binits{V.~C.}}
(\byear{2020}).
\btitle{Consistent feature selection for neural networks via Adaptive Group
  Lasso}.
\bjournal{ArXiv}.
\end{barticle}
\endbibitem

\bibitem[\protect\citeauthoryear{Ji, Kollár and Shiffman}{1992}]{RN496}
\begin{barticle}[author]
\bauthor{\bsnm{Ji},~\bfnm{Shanyu}\binits{S.}},
  \bauthor{\bsnm{Kollár},~\bfnm{János}\binits{J.}} \AND
  \bauthor{\bsnm{Shiffman},~\bfnm{Bernard}\binits{B.}}
(\byear{1992}).
\btitle{A global lojasiewicz inequality for algebraic varieties}.
\bjournal{Transactions of the American Mathematical Society}
\bvolume{329}
\bpages{813-818}.
\end{barticle}
\endbibitem

\bibitem[\protect\citeauthoryear{Karakida, Akaho and Amari}{2019}]{RN545}
\begin{barticle}[author]
\bauthor{\bsnm{Karakida},~\bfnm{Ryo}\binits{R.}},
  \bauthor{\bsnm{Akaho},~\bfnm{Shotaro}\binits{S.}} \AND
  \bauthor{\bsnm{Amari},~\bfnm{Shun‐ichi}\binits{S.}}
(\byear{2019}).
\btitle{Pathological Spectra of the Fisher Information Metric and Its Variants
  in Deep Neural Networks}.
\bjournal{Neural Computation}
\bvolume{33}
\bpages{2274-2307}.
\end{barticle}
\endbibitem

\bibitem[\protect\citeauthoryear{Kohler, Krzyżak and Langer}{2022}]{RN360}
\begin{barticle}[author]
\bauthor{\bsnm{Kohler},~\bfnm{M.}\binits{M.}},
  \bauthor{\bsnm{Krzyżak},~\bfnm{A.}\binits{A.}} \AND
  \bauthor{\bsnm{Langer},~\bfnm{S.}\binits{S.}}
(\byear{2022}).
\btitle{Estimation of a Function of Low Local Dimensionality by Deep Neural
  Networks}.
\bjournal{IEEE Transactions on Information Theory}
\bvolume{68}
\bpages{4032-4042}.
\bdoi{10.1109/TIT.2022.3146620}
\end{barticle}
\endbibitem

\bibitem[\protect\citeauthoryear{Lafferty and Wasserman}{2008}]{RN306}
\begin{barticle}[author]
\bauthor{\bsnm{Lafferty},~\bfnm{John}\binits{J.}} \AND
  \bauthor{\bsnm{Wasserman},~\bfnm{Larry}\binits{L.}}
(\byear{2008}).
\btitle{Rodeo: Sparse, greedy nonparametric regression}.
\bjournal{The Annals of Statistics}
\bvolume{36}
\bpages{28-63}.
\bdoi{10.1214/009053607000000811}
\end{barticle}
\endbibitem

\bibitem[\protect\citeauthoryear{Lederer}{2023}]{RN321}
\begin{barticle}[author]
\bauthor{\bsnm{Lederer},~\bfnm{Johannes}\binits{J.}}
(\byear{2023}).
\btitle{Statistical guarantees for sparse deep learning}.
\bjournal{AStA Advances in Statistical Analysis}.
\bdoi{10.1007/s10182-022-00467-3}
\end{barticle}
\endbibitem

\bibitem[\protect\citeauthoryear{Li, Wang and Ding}{2023}]{RN340}
\begin{barticle}[author]
\bauthor{\bsnm{Li},~\bfnm{Geng}\binits{G.}},
  \bauthor{\bsnm{Wang},~\bfnm{G.}\binits{G.}} \AND
  \bauthor{\bsnm{Ding},~\bfnm{Jie}\binits{J.}}
(\byear{2023}).
\btitle{Provable Identifiability of Two-Layer ReLU Neural Networks via LASSO
  Regularization}.
\bjournal{IEEE Transactions on Information Theory}
\bvolume{69}
\bpages{5921-5935}.
\end{barticle}
\endbibitem

\bibitem[\protect\citeauthoryear{Lin and Zhang}{2006}]{RN304}
\begin{barticle}[author]
\bauthor{\bsnm{Lin},~\bfnm{Yi}\binits{Y.}} \AND
  \bauthor{\bsnm{Zhang},~\bfnm{Hao~Helen}\binits{H.~H.}}
(\byear{2006}).
\btitle{Component selection and smoothing in multivariate nonparametric
  regression}.
\bjournal{The Annals of Statistics}
\bvolume{34}
\bpages{2272-2297}.
\bdoi{10.1214/009053606000000722}
\end{barticle}
\endbibitem

\bibitem[\protect\citeauthoryear{Luo and Halabi}{2023}]{RN369}
\begin{barticle}[author]
\bauthor{\bsnm{Luo},~\bfnm{Bin}\binits{B.}} \AND
  \bauthor{\bsnm{Halabi},~\bfnm{Susan}\binits{S.}}
(\byear{2023}).
\btitle{Sparse-Input Neural Network using Group Concave Regularization}.
\bjournal{ArXiv}
\bvolume{abs/2307.00344}.
\end{barticle}
\endbibitem

\bibitem[\protect\citeauthoryear{Neyshabur, Tomioka and Srebro}{2015}]{RN503}
\begin{binproceedings}[author]
\bauthor{\bsnm{Neyshabur},~\bfnm{Behnam}\binits{B.}},
  \bauthor{\bsnm{Tomioka},~\bfnm{Ryota}\binits{R.}} \AND
  \bauthor{\bsnm{Srebro},~\bfnm{Nathan}\binits{N.}}
(\byear{2015}).
\btitle{Norm-Based Capacity Control in Neural Networks}.
In \bbooktitle{Annual Conference Computational Learning Theory}
\bpages{1376-1401}.
\end{binproceedings}
\endbibitem

\bibitem[\protect\citeauthoryear{Ravikumar et~al.}{2007}]{RN303}
\begin{binproceedings}[author]
\bauthor{\bsnm{Ravikumar},~\bfnm{Pradeep}\binits{P.}},
  \bauthor{\bsnm{Liu},~\bfnm{Han}\binits{H.}},
  \bauthor{\bsnm{Lafferty},~\bfnm{John~D.}\binits{J.~D.}} \AND
  \bauthor{\bsnm{Wasserman},~\bfnm{Larry~A.}\binits{L.~A.}}
(\byear{2007}).
\btitle{SpAM: Sparse Additive Models}.
In \bbooktitle{Neural Information Processing Systems}.
\end{binproceedings}
\endbibitem

\bibitem[\protect\citeauthoryear{Rosasco et~al.}{2013}]{RN278}
\begin{barticle}[author]
\bauthor{\bsnm{Rosasco},~\bfnm{Lorenzo}\binits{L.}},
  \bauthor{\bsnm{Villa},~\bfnm{Silvia}\binits{S.}},
  \bauthor{\bsnm{Mosci},~\bfnm{Sofia}\binits{S.}},
  \bauthor{\bsnm{Santoro},~\bfnm{Matteo}\binits{M.}} \AND
  \bauthor{\bsnm{Verri},~\bfnm{Alessandro}\binits{A.}}
(\byear{2013}).
\btitle{Nonparametric sparsity and regularization}.
\bjournal{Journal of Machine Learning Research}
\bvolume{14}
\bpages{1665-1714}.
\end{barticle}
\endbibitem

\bibitem[\protect\citeauthoryear{Sagun, Bottou and LeCun}{2017}]{Sagun2016}
\begin{barticle}[author]
\bauthor{\bsnm{Sagun},~\bfnm{Levent}\binits{L.}},
  \bauthor{\bsnm{Bottou},~\bfnm{L{\'e}on}\binits{L.}} \AND
  \bauthor{\bsnm{LeCun},~\bfnm{Yann}\binits{Y.}}
(\byear{2017}).
\btitle{Eigenvalues of the Hessian in Deep Learning: Singularity and Beyond}.
\bjournal{ArXiv}.
\end{barticle}
\endbibitem

\bibitem[\protect\citeauthoryear{Schmidt-Hieber}{2020}]{RN325}
\begin{barticle}[author]
\bauthor{\bsnm{Schmidt-Hieber},~\bfnm{Johannes}\binits{J.}}
(\byear{2020}).
\btitle{Nonparametric regression using deep neural networks with ReLU
  activation function}.
\bjournal{The Annals of Statistics}
\bvolume{48}
\bpages{1875-1897}.
\bdoi{10.1214/19-AOS1875}
\end{barticle}
\endbibitem

\bibitem[\protect\citeauthoryear{Taheri, Xie and Lederer}{2021}]{RN323}
\begin{barticle}[author]
\bauthor{\bsnm{Taheri},~\bfnm{Mahsa}\binits{M.}},
  \bauthor{\bsnm{Xie},~\bfnm{Fang}\binits{F.}} \AND
  \bauthor{\bsnm{Lederer},~\bfnm{Johannes}\binits{J.}}
(\byear{2021}).
\btitle{Statistical Guarantees for Regularized Neural Networks}.
\bjournal{Neural networks}
\bvolume{142}
\bpages{148-161}.
\end{barticle}
\endbibitem

\bibitem[\protect\citeauthoryear{Tibshirani}{1996}]{RN298}
\begin{barticle}[author]
\bauthor{\bsnm{Tibshirani},~\bfnm{Robert}\binits{R.}}
(\byear{1996}).
\btitle{Regression Shrinkage and Selection Via the Lasso}.
\bjournal{Journal of the Royal Statistical Society: Series B (Methodological)}
\bvolume{58}
\bpages{267-288}.
\bdoi{https://doi.org/10.1111/j.2517-6161.1996.tb02080.x}
\end{barticle}
\endbibitem

\bibitem[\protect\citeauthoryear{Wainwright}{2019}]{RN299}
\begin{bbook}[author]
\bauthor{\bsnm{Wainwright},~\bfnm{Martin~J.}\binits{M.~J.}}
(\byear{2019}).
\btitle{High-Dimensional Statistics: A Non-Asymptotic Viewpoint}.
\bseries{Cambridge Series in Statistical and Probabilistic Mathematics}.
\bpublisher{Cambridge University Press}.
\bdoi{DOI: 10.1017/9781108627771}
\end{bbook}
\endbibitem

\bibitem[\protect\citeauthoryear{Zhang}{2010}]{RN482}
\begin{barticle}[author]
\bauthor{\bsnm{Zhang},~\bfnm{Cun-Hui}\binits{C.-H.}}
(\byear{2010}).
\btitle{Nearly unbiased variable selection under minimax concave penalty}.
\bjournal{The Annals of Statistics}
\bvolume{38}
\bpages{894-942}.
\end{barticle}
\endbibitem

\end{thebibliography}

\end{document}